\documentclass[12pt, reqno]{amsart}
     \makeatletter
     \def\section{\@startsection{section}{1}%
     \z@{.7\linespacing\@plus\linespacing}{.5\linespacing}%
     {\bfseries
     \centering
     }}
     \def\@secnumfont{\bfseries}
     \makeatother
\usepackage{fullpage}
\usepackage{amsmath,amsthm,amssymb,latexsym,amscd}
\usepackage{wrapfig}
\usepackage{amssymb}
\usepackage[final]{hyperref}
\usepackage{tikz}
\usepackage{hyperref}
 \usepackage{amsrefs}
\usepackage{graphicx}
 \usepackage{wrapfig}
\usepackage{enumerate}
\usepackage{mathtools}
\usepackage{subfigure}

\newtheorem{theorem}{Theorem}[section]
\newtheorem{lemma}{Lemma}[section]
\newtheorem{prop}{Proposition}[section]

\theoremstyle{definition}
\newtheorem{definition}{Definition}[section]

\numberwithin{equation}{section}

\usepackage{graphicx}
\usepackage{amssymb}
\usepackage[final]{hyperref}
\usepackage{tikz}
 
 \usepackage{wrapfig}

\newcommand{\mcx}{{\mathcal X}}
\newcommand{\mbr}{{\mathbb R}}

\newcommand{\mbe}{{\mathbb E}}

\newcommand{\la}{{ \langle}}

\newcommand{\ra}{{ \rangle}}

\newcommand{\defeq}{\vcentcolon=}

\begin{document}

\title[The Gaussian Radon Transform and Machine Learning]{The Gaussian Radon Transform and Machine Learning}



\author{Irina Holmes}
\address{Department of Mathematics \\
 Louisiana State University \\
Baton Rouge, LA 70803 \\
e-mail: \sl irina.c.holmes@gmail.com}
\thanks{Irina Holmes is a Dissertation Year Fellow at Louisiana State University, Department of Mathematics; {\sl irina.c.holmes@gmail.com}}

\author{Ambar N.~Sengupta}
\address{Department of Mathematics \\
 Louisiana State University \\
Baton Rouge, LA 70803 \\
e-mail: \sl ambarnsg@gmail.com}
\thanks{Ambar Niel Sengupta  is Hubert Butts Professor of Mathematics  at Louisiana State University; {\sl ambarnsg@gmail.com}}


\date{March 2014}

\dedicatory{}

\begin{abstract}
There has  been  growing  recent interest in probabilistic interpretations of kernel-based methods as well as learning in Banach spaces.   The absence of a useful  Lebesgue measure on an infinite-dimensional   reproducing kernel Hilbert space  is a serious obstacle for such stochastic models. We propose an estimation model for the ridge regression problem within the framework of abstract Wiener spaces and show how the support vector machine solution to such problems can be interpreted in terms of the Gaussian Radon transform.
\end{abstract}

\maketitle

\section{Introduction}\label{S:Intro}

A central task in machine learning is the prediction of unknown values based on `learning' from a given set of outcomes. More precisely, suppose $\mathcal{X}$ is a non-empty set, the {\em input space}, and
	\begin{equation}
	D = \{ (p_1, y_1), (p_2, y_2), \ldots, (p_n, y_n)\} \subset \mathcal{X} \times \mathbb{R}
	\end{equation}
is a finite collection of input values $p_j$ together with their corresponding real outputs $y_j$. The goal is to predict the value $y$ corresponding to a yet unobserved input value $p \in \mathcal{X}$. The fundamental assumption of kernel-based methods such as support vector machines (SVM) is that the predicted value is given by $\hat{f}(p)$, where the {\em decision} or {\em prediction function} $\hat{f}$ belongs to a reproducing kernel Hilbert space (RKHS) $H$ over $\mathcal{X}$, corresponding to  a positive definite function $K : \mathcal{X}\times\mathcal{X}\rightarrow\mathbb{R}$ (see Section \ref{Ss:RKHS}). In particular, in {\em ridge regression}, the predictor function $\hat{f}_{\lambda}$ is the solution to the minimization problem:
	\begin{equation} \label{E:RegRegr}
	\hat{f}_{\lambda} = \text{arg}\min_{f\in H} \left( \sum_{j=1}^n (y_j - f(p_j))^2 + \lambda\|f\|^2 \right),
	\end{equation}
where $\lambda > 0$ is a {\em regularization parameter} and $\|f\|$ denotes the norm of $f$ in $H$.   The regularization  term $ \lambda\|f\|^2$  penalizes functions that `overfit' the training data.

The minimization problem  \eqref{E:RegRegr} has a unique solution, given by a linear combination of the functions $K_{p_j} = K(p_j, \cdot)$. Specifically:
	\begin{equation}\label{E:RegRegrSol}
	\hat{f}_{\lambda} = \sum_{j=1}^n c_jK_{p_j}
	\end{equation}
where $c\in\mathbb{R}^n$ is given by:
	\begin{equation}\label{E:RegRegrSola}
	c = (K_D + \lambda I_n)^{-1}y
	\end{equation}
where $K_D \in \mathbb{R}^{n\times n}$ is the matrix with entries $[K_D]_{i j} = K(p_i, p_j)$, $I_n$ is the identity matrix of size $n$ and $y=(y_1,\ldots, y_n)\in\mathbb{R}^n$ is the vector of outputs. {\em We present a geometrical proof of this classic result in Theorem} \ref{T:svmminim}.

Recently there has been a surge in interest in probabilistic interpretations of kernel-based learning methods; see for instance \cites{Sollich, Muk, Aravkin, Ozertem, Zhang, Huszar}. As we shall see below, there is a Bayesian interpretation and a stochastic process approach to the ridge regression problem when the RKHS is finite-dimensional, but many RKHSs used in practice are infinite-dimensional. The absence of Lebesgue measure on infinite-dimensional Hilbert spaces poses a roadblock to such interpretations in this case. In this paper:
	\begin{itemize}
	\item We show that there is a valid stochastic interpretation to the ridge regression problem in the infinite-dimensional case, by working within the framework of abstract Wiener spaces and Gaussian measures instead of working with the Hilbert space alone. Specifically, we show that the ridge regression solution may be obtained in terms of a conditional expectation and the Gaussian Radon transform.
	\item We show that, within the more traditional spline setting, the element of $H$ of minimal norm that satisfies $f(p_j) = y_j$ for $1 \leq j \leq n$ can also be expressed in terms of the Gaussian Radon transform.
	\item We propose a way to use the Gaussian Radon transform for a broader class of prediction problems. Specifically, this method could potentially be used to not only predict a particular output $f(p)$ of a future input $p$, but to predict a function of future outputs; for instance, one could be interested in predicting the maximum or minimum value over an interval of future outputs.
	\item In this work we start with a RKHS $H$ and complete it with respect to a measurable norm to obtain a Banach space $B$. However, the space $B$ does not necessarily consist of functions. In light of recent interest in reproducing kernel Banach spaces, which do consist of functions, we propose a method to realize $B$ as a space of functions.
	\end{itemize}

\subsection{Probabilistic Interpretations to the Ridge Regression Problem in Finite Dimensions}

Let us first consider a Bayesian perspective: suppose that our parameter space is an RKHS $H$ with reproducing kernel $K : \mathcal{X} \times \mathcal{X} \rightarrow \mathbb{R}$. If $H$ is finite-dimensional, we take standard Gaussian measure on $H$ as our \emph{prior distribution}:
	\begin{equation*}
	\rho(f) = \frac{1}{(2\pi)^{d/2}} e^{-\frac{1}{2}\|f\|^2} \text{, for all } f \in H,
	\end{equation*}
where $d$ is the dimension of $H$. Let $\tilde{f}$ denote, for every $f \in H$, the continuous linear functional $\la f, \cdot\ra$ on $H$. With respect to standard Gaussian measure, every $\tilde{f}$ is normally distributed with mean $0$ and variance $\|f\|^2$, and Cov$(\tilde{f}, \tilde{g}) = \la f, g\ra$. Now recall that $H$ contains the functions $K_p = K(p, \cdot)$ on $\mathcal{X}$, and $f(p) = \la f, K_p\ra$ for all $f \in H$, $p \in \mathcal{X}$. Then $\{\tilde{K}_p\}_{p\in\mathcal{X}}$ is a centered Gaussian process on $H$ with covariance function $K$:  ${\rm Cov}(\tilde{K}_p, \tilde{K}_{p'}) = K(p,p')$ for all $p, p' \in \mathcal{X}$. If we are given the training data $D = \{(p_1, y_1), \ldots, (p_n, y_n)\} \subset \mathcal{X} \times \mathbb{R}$ and we assume the measurement contains some error we would like to model as Gaussian noise, we choose an orthonormal set $\{e_1, \ldots, e_n\} \subset H$ such that $e_j \perp K_{p_i}$ for every $1 \leq i, j \leq n$. Then
	$$\{\tilde{K}_{p_j}\}_{1\leq j \leq n} \text{ and } \{\tilde{e}_j\}_{1\leq j \leq n}$$
are independent centered Gaussian processes on $H$, and for every $f \in H$ we model our data as arising from the event:
	$$ \tilde{y}_j = \tilde{K}_{p_j}(f) + \sqrt{\lambda}\tilde{e}_j\text{, for all } 1\leq j \leq n,$$
where $\lambda > 0$ is a fixed parameter. Then for every $f \in H$, $\{\tilde{y}_1, \ldots, \tilde{y}_n\}$ are independent Gaussian random variables on $H$ with mean $f(p_j)$ and variance $\lambda$ for every $1 \leq j \leq n$, which gives rise to the statistical model of probability distributions on $\mathbb{R}^n$:
	$$\rho_{\lambda}(x | f) = \prod_{j=1}^n \frac{1}{\sqrt{2\pi\lambda}} e^{-\frac{1}{2\lambda}(x_j - f(p_j))^2},$$
for every $x = (x_1, \ldots, x_n) \in \mbr^n$. Replacing $x$ with the vector $y = (y_1, \ldots, y_n)$ of observed values, the \emph{posterior distribution} resulting from this is proportional to:
	$$e^{-\frac{1}{2\lambda}[\sum_{j=1}^n (y_j - f(p_j))^2 + \lambda\|f\|^2]}.$$
Therefore finding the \emph{maximum a posteriori} (MAP) estimator in this situation amounts exactly to finding the solution to the ridge regression problem in \eqref{E:RegRegr}.

Clearly, this Bayesian approach depends on $H$ being finite-dimensional; however, reproducing kernel Hilbert spaces used in practice are often infinite-dimensional (such as those arising from Gaussian RBF kernels). The ridge regression SVM previously discussed goes through regardless of the dimensionality of the RKHS, and there is still a need for a valid stochastic interpretation of the infinite-dimensional case. We now explore another stochastic approach to ridge regression, which is equivalent to the Bayesian one, but which can be carried over in a sense to the infinite-dimensional case, as we shall see later. Suppose again that $H$ is a finite-dimensional RKHS over $\mcx$, with reproducing kernel $K$, and equipped with standard Gaussian measure. Recall that $\{\tilde{K}_p\}_{p\in\mcx}$ is then a centered Gaussian process on $H$ with covariance function $K$. If we assume that the data arises from some unknown function in $H$, then the relationship $f(p) = \tilde{K}_p(f)$ suggests that the random variable $\tilde{K}_p$ is a good model for the outputs. Moreover, the training data $D=\{ (p_1, y_1), \ldots, (p_n, y_n) \}$ provides some previous knowledge of the random variables $\tilde{K}_{p_j}$, which we can use to refine our estimation of $\tilde{K}_p$ by taking conditional expectations. In other words, our first guess would be to estimate the output of a future input $p\in\mcx$ by:
	\begin{equation} \label{E:MLFirstInstinct}
	\mbe[\tilde{K}_p | \tilde{K}_{p_1} = y_1, \ldots, \tilde{K}_{p_n} = y_n].
	\end{equation}
But if we want to include some possible noise in the measurements, we would like to have a centered Gaussian process $\{\epsilon_1, \ldots, \epsilon_n\}$ on $H$, with covariance function Cov$(\epsilon_i, \epsilon_j) = \lambda \delta_{i,j}$ for some parameter $\lambda > 0$, which is also independent of $\{K_{p_j}\}_{1\leq j \leq n}$.

Let us fix $p\in\mcx$, a `future' input whose output we would like to predict. To take measurement error into account, we  choose again an orthonormal set $\{e_1, \ldots, e_n\}\subset H$ such that:
	\begin{equation*}
	\{e_1, \ldots, e_n\} \subset [\text{span}\{K_{p_1}, \ldots, K_{p_n}, K_p\}]^{\perp},
	\end{equation*}
and $\lambda > 0$, and set:
	\begin{equation*}
	y_j = \tilde{K}_{p_j} + \sqrt{\lambda}\tilde{e}_j \text{, for all } 1\leq j \leq n.
	\end{equation*}
Then we estimate the output $\hat{y}(p)$ as the conditional expectation:
	\begin{equation*}
	\hat{y}(p) = \mbe[\tilde{K}_p | \tilde{K}_{p_j} + \sqrt{\lambda}\tilde{e}_j = y_j, 1 \leq j \leq n].
	\end{equation*}
As shown in Lemma \ref{L:condExpGauss} below:
	\begin{equation*}
	\hat{y}(p) = a_1y_1 + \ldots + a_ny_n,
	\end{equation*}
where $a=(a_1, \ldots, a_n) \in \mbr^n$ is:
	\begin{equation*}
	a = A^{-1}\left[\text{Cov}(\tilde{K}_p, \tilde{K}_{p_j} + \sqrt{\lambda}\tilde{e}_j)\right]_{1\leq j \leq n},
	\end{equation*}
with:
	\begin{equation*}
	[A]_{i,j} = [K_D + \lambda I_n]_{i,j},
	\end{equation*}
for all $1\leq i, j \leq n$, with $K_D$ being the $n\times n$ matrix with entries given by $K(p_i,p_j)$. Moreover:
	\begin{equation*}
	\text{Cov}(\tilde{K}_p, \tilde{K}_{p_j} + \sqrt{\lambda}\tilde{e}_j) = \la K_p, K_{p_j}\ra = K_{p_j}(p).
	\end{equation*}
Note  that this last relationship is why we required that $\{e_1, \ldots, e_n\}$ also be orthogonal to $K_p$. This yields:
	\begin{equation*}
	\hat{y}(p) = \sum_{j=1}^n [(K_D + \lambda I_n)^{-1}y]_j K_{p_j}(p),
	\end{equation*}
showing that the prediction $\hat{y}(p)$ is precisely the ridge regression solution $\hat{f}_{\lambda}(p)$ in \eqref{E:RegRegrSol}.

Our goal in this paper is to show that the Gaussian process approach does go through in infinite-dimensions and, moreover, that it remains equivalent to the SVM solution. Since the SVM solution in \eqref{E:RegRegrSol} is contained in a finite-dimensional subspace, the possibly infinite dimensionality of $H$ is less of a problem in this setting. However, if we want to predict based on a stochastic model for $f$ and $H$ is infinite-dimensional, the absence of Lebesgue measure becomes a significant problem. In particular, we cannot have the desired Gaussian process $\{\tilde{K}_p:p\in\mathcal{X}\}$ on $H$ itself.  The approach we propose is to work within the framework of abstract Wiener spaces, introduced by L. Gross in the celebrated work \cite{Gr}. The concept of abstract Wiener space was born exactly from this need for a ``standard Gaussian measure'' in infinite dimensions and has become a standard framework in infinite-dimensional analysis.  We outline the basics of this theory in Section \ref{ss:gmbs} and showcase the essentials of the classical Wiener space, as  the ``original'' special case of an abstract Wiener space, in Section \ref{ss:CWS}.  

We construct a Gaussian measure with the desired properties on a larger Banach space $B$ that contains $H$ as a dense subspace;  the geometry of this measure is dictated by the inner-product on $H$. This construction   is presented in Section \ref{ss:cons}. We then show in Section \ref{S:gml} how the ridge regression learning problem outlined above can be understood in terms of the Gaussian Radon transform. The Gaussian Radon transform associates to a function $f$ on $B$ the function $Gf$, defined on the set of closed affine subspaces of $H$, whose value $Gf(L)$ on a closed affine subspace $L\subset H$ is the integral $\int f\,d\mu_L$, where $\mu_L$ is a Gaussian measure on the closure $\overline{L}$ of $L$ in $B$ obtained from the given Gaussian measure on $B$. This is explained in more detail in Section \ref{S:backg}. For more on the Gaussian Radon transform we refer to the works \cites{BecGR2010, BecSen2012, Hol2013, HolSen2012, MS}. 

Another area that has recently seen strong activity is learning in Banach spaces; see for instance \cites{Xu, learning, Der}. Of particular interest are reproducing kernel Banach spaces, introduced in \cite{Xu}, which are special Banach spaces whose elements are functions. The Banach space $B$ we use in Section \ref{S:gml} is a completion of a reproducing kernel Hilbert space, but does not directly consist of functions. A notable exception is the classical Wiener space, where the Banach space is $C[0,1]$. In Section \ref{S:Basfunct} we address this issue and propose a realization of $B$ as a space of functions.

Finally, in Appendix \ref{S:geomform} we present a geometric view. First we present a geometric proof  of the representer theorem for the SVM minimization problem and then describe the  relationship with the Gaussian Radon transform in geometric terms.

\section{Background}\label{S:Background}

\subsection{Realizing covariance structures with Banach spaces}\label{S:HilbGaussGauss}

In this section we construct a Gaussian measure on a Banach space along with random variables defined on this space for which the covariance structure is specified in advance.  In more detail, suppose ${\mathcal X}$ is a non-empty set and 
	\begin{equation}\label{E:defK}
	K: {\mathcal X}\times{\mathcal X}\to\mbr
	\end{equation}
a function that is symmetric and positive definite (in the sense that the matrix $[K(p,q)]_{p,q\in{\mathcal X}}$ is symmetric and positive definite); we will construct a measure $\mu$ on a certain Banach space $B$ along with a family of Gaussian random variables $\tilde{K}_p$,  with $p$ running over $ {\mathcal X}$, such that  
	\begin{equation}\label{E:KpqCov} 
	K(p,q)={\rm Cov}(\tilde{K}_p, \tilde{K}_q) 
	\end{equation}
for all $p, q\in{\mathcal X}$.  A well-known choice for $K(p,q)$, for $p,q\in \mbr^d$,  is given by
	$$e^{-\frac{|p-q|^2}{2s}},$$
where $s>0$ is a scale parameter.
 
The strategy is as follows: first we construct a Hilbert space $H$ along with elements $K_p\in H$, for each $p\in {\mathcal X}$, for which $\la  K_p,  K_q\ra=K(p,q)$ for all $p, q\in {\mathcal X}$. Next we describe how to obtain a Banach space $B$, equipped with a Gaussian measure, along with random variables $\tilde{K}_p$  that have the required covariance structure \eqref{E:KpqCov}. The first step is   a standard result for reproducing kernel Hilbert spaces.

\subsubsection{Constructing a Hilbert space from a covariance structure}  \label{Ss:RKHS}

For this we simply quote the well-known Moore-Aronszajn theorem  (see, for example, Chapter 4 of Steinwart  \cite{Ingo}).
 
 \begin{theorem}\label{T: covHilb} Let ${\mathcal X}$ be a non-empty set and $K: {\mathcal X}\times {\mathcal X}\to\mbr$ a function for which the matrix $[K(p,q)]_{p,q\in {\mathcal X}}$ is symmetric and positive definite:
 \begin{itemize}
 \item[(i)] $K(p,q)=K(q,p)$ for all $p, q\in {\mathcal X}$, and
 \item[(ii)] $\sum_{j,k=1}^Nc_jc_kK(p_j,p_k)\geq 0$ holds for all integers $N \geq 1$, all points $p_1, \ldots, p_N \in {\mathcal X}$, and all $c_1,\ldots, c_N \in \mbr$.
 \end{itemize}
 Then there is a unique Hilbert space $H$ consisting of real-valued functions defined on the set ${\mathcal X}$,  containing the functions $K_p=K(p,\cdot)$ for all $p\in {\mathcal X}$, with the inner product on $H$ being such that
 \begin{equation}\label{E:reker}
 f(p)=\la  K_p, f\ra\qquad\hbox{for all $f\in H$ and $p\in {\mathcal X}$.}
 \end{equation}
 Moreover, the linear span of $\{ K_p \,: \,p\in {\mathcal X}\}$ is dense in $H$.
 \end{theorem} 
 
The Hilbert space $H$ is called the {\em reproducing kernel Hilbert space} (RKHS) over $\mathcal{X}$ with {\em reproducing kernel} $K$.
 
For the mapping
\begin{equation}\label{E:Phipq1}
\Phi:{\mathcal X}\to H:p\mapsto \Phi(p) = K_p,
\end{equation}
 known as the {\em canonical feature map}, we have
	\begin{equation}\label{E:Phipq}
	|\!|\Phi(p)-\Phi(q)|\!|^2=K(p,p)-2K(p,q)+K(q,q)
	\end{equation}
for all $p, q\in {\mathcal X}$. Hence if ${\mathcal X}$ is a topological space and $K$ is continuous then so is the function $(p,q)\mapsto |\!|\Phi(p)-\Phi(q)|\!|^2$ given in (\ref{E:Phipq}), and the value of this being $0$ when $p=q$, it follows that $p\mapsto\Phi(p)$ is continuous. In particular, if ${\mathcal X}$ has a countable dense subset then so does the image $\Phi({\mathcal X})$, and since this spans a dense subspace of $H$ it follows that $H$ is separable.

\subsubsection{Gaussian measures on Banach spaces}\label{ss:gmbs} The theory of abstract Wiener spaces developed by Gross  \cite{Gr}  provides the machinery for constructing a Gaussian measure on a Banach space $B$ obtained by completing a given Hilbert space $H$ using a special type of norm $|\cdot|$ called a {\em measurable norm}. 
Conversely,  according to a fundamental result of Gross, every centered  non-degenerate Gaussian measure on a real separable Banach space arises in this way from completing an underlying Hilbert space called the {\em Cameron-Martin space}. We work here with real Hilbert spaces as a complex structure plays no role in the Gaussian measure in this context.  
  
\begin{definition}
A norm $|\cdot|$  on a real separable Hilbert space $H$ is said to be a {\em measurable norm} provided that for any $\epsilon > 0$, there is a finite-dimensional subspace $F_{\epsilon}$ of $H$ such that:	
	\begin{equation}\label{E:hepsilon}
	\gamma_F \left\{ h \in F : |h| > \epsilon \right\} < \epsilon
	\end{equation}
for every finite-dimensional subspace $F$ of $H$ with $F \perp F_{\epsilon}$, where $\gamma_F$ denotes standard Gaussian measure on $F$. Figure \ref{FIG:aws}(a) illustrates this notion.
\end{definition}

\begin{figure} 
\begin{center}
\subfigure[A measurable norm $|\cdot|$.]{\includegraphics[width=7cm]{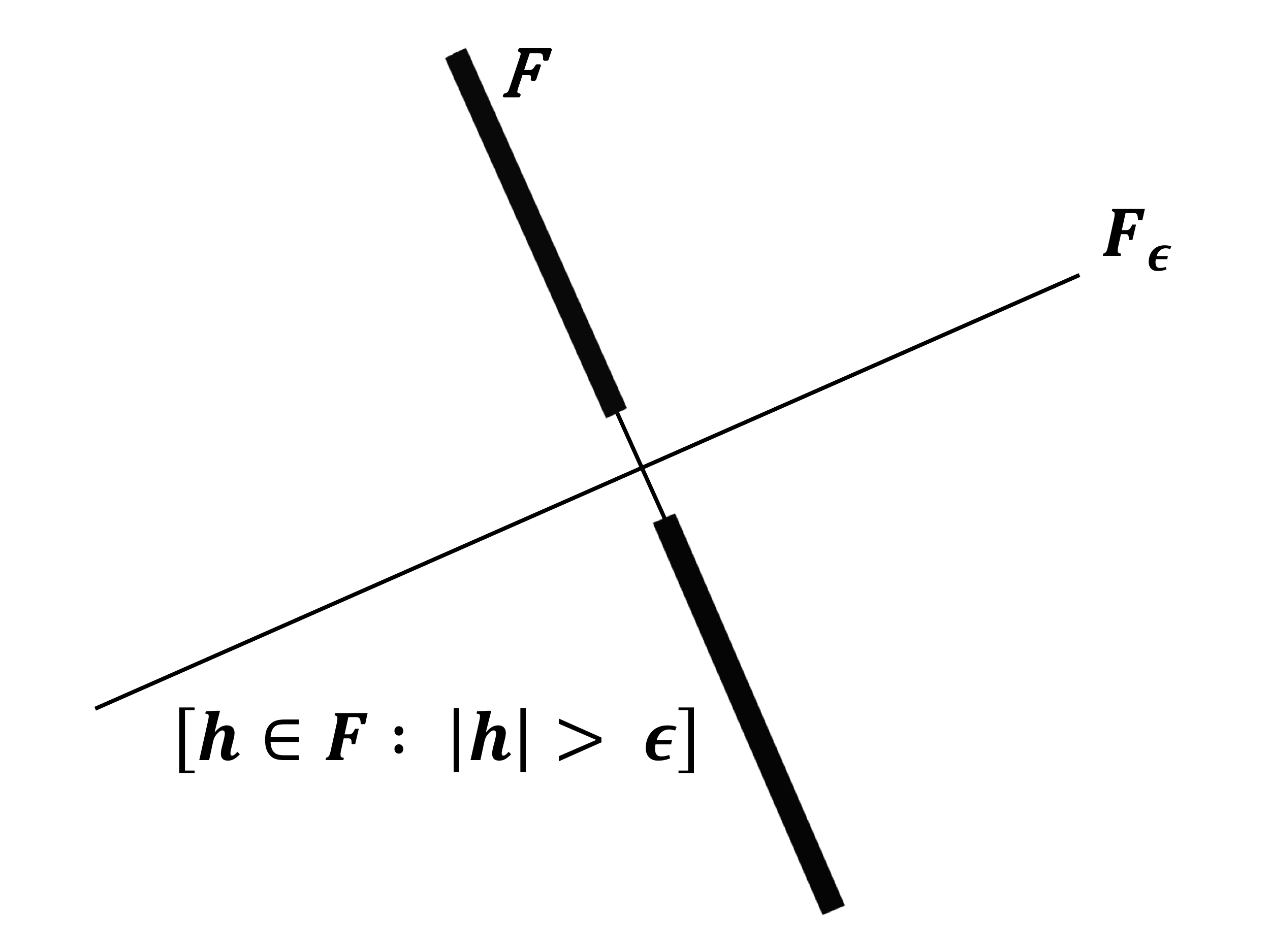}}
\hfill
\subfigure[Abstract Wiener Space.]{\includegraphics[width=7cm]{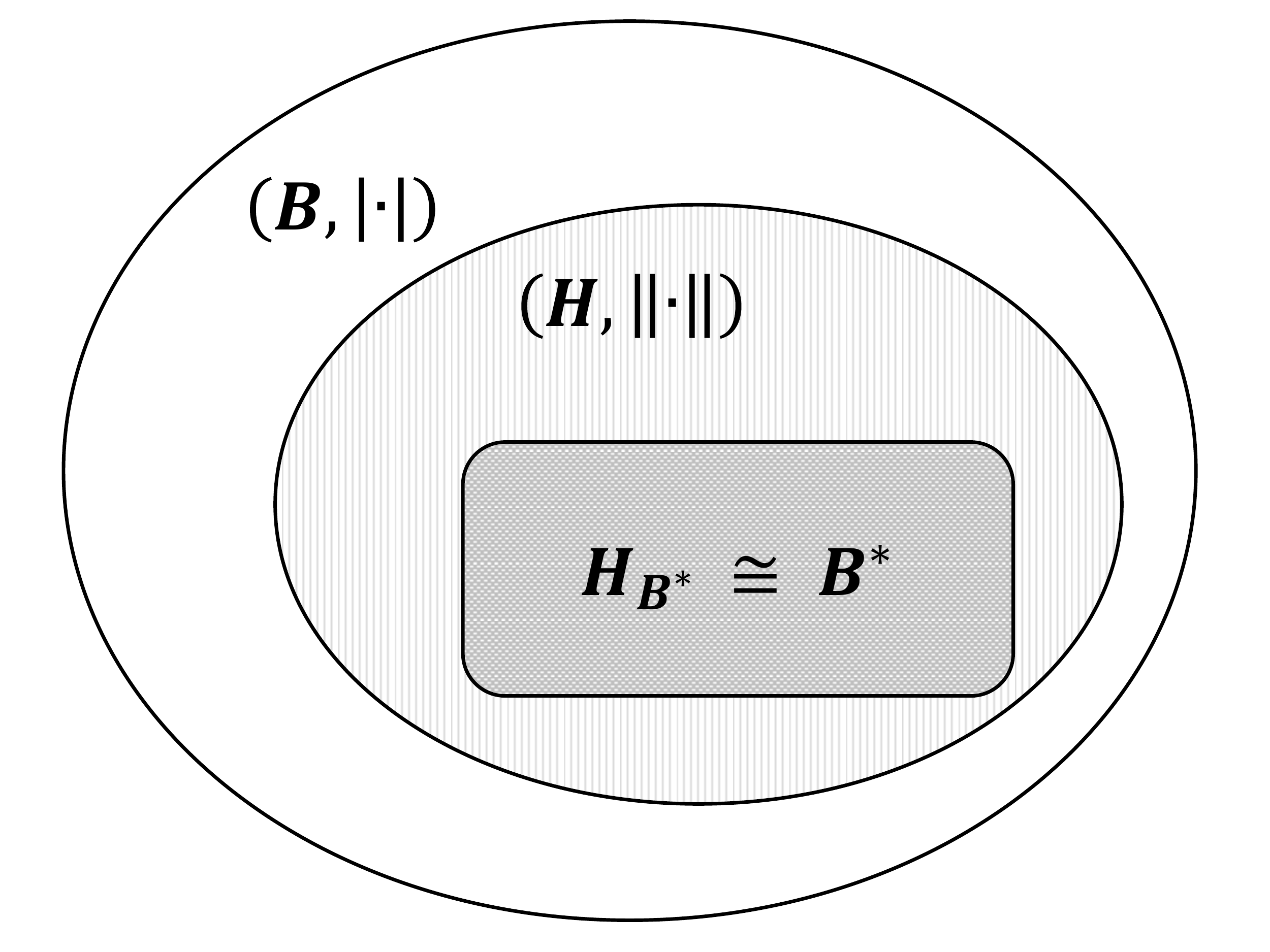}}
\caption{}
\label{FIG:aws}
\end{center}
\end{figure}

We denote the norm on $H$ arising from the inner-product $\la\cdot,\cdot\ra$ by $|\!|\cdot|\!|$, which is not to be confused with a measurable norm $|\cdot|$. Here are three facts about measurable norms (for proofs see Gross \cite{Gr}, Kuo \cite{Ku1} or Eldredge \cite{Nate}):  

	\begin{enumerate}
	\item A measurable norm is always weaker than the original norm: there is $c > 0$ such that:
		\begin{equation} \label{E:msnormwkr}
		|h| \leq c \|h\|,
		\end{equation}
		for all $h \in H$.

	\item If $H$ is infinite-dimensional, the original Hilbert norm $\|\cdot\|$ is not a measurable norm.
	
	\item If  $A$ is an injective Hilbert-Schmidt operator on $H$ then
		\begin{equation} \label{E:HilbSchmidt}
		|h| = \|Ah\| \text{, for all } h \in H
		\end{equation}
	specifies a measurable  norm on $H$.
	\end{enumerate}

Henceforth we denote by $B$ the Banach space obtained by completion of $H$ with respect to $|\cdot|$.  
	
Any element $f\in H$  gives a linear functional
	$$H\to\mbr: g\mapsto \la f,g\ra$$
	that is continuous with respect to the norm $\|\cdot\|$. Let
	$$H_{B^*}$$
	be the subspace of $H$ consisting of all $f$ for which $g\mapsto\la f,g\ra$ is continuous with respect to the norm $ |\cdot |$. Figure \ref{FIG:aws}(b) illustrates the relationships between $H$, $B$, and $H_{B^*}$.

	 By fact (1) above, any linear functional on $H$ that is continuous with respect to $|\cdot|$ is also continuous with respect to $\|\cdot\|$ and hence is of the form $\la f,\cdot\ra$ for a unique $f\in H$ by the traditional Riesz theorem. Thus $H_{B^*}$ consists precisely of those $f\in H$ for which the linear functional $\la f,\cdot\ra$ {\em is the restriction to $H$ of a (unique) continuous linear functional on the Banach space $B$}. We denote this extension of $\la f,\cdot\ra$ to an element of $B^*$ by $I(f)$:
	\begin{eqnarray}
	I(f) &=& \text{the continuous linear functional on $B$} \nonumber\\
	&& \text{agreeing with }\la f,\cdot\ra \text{ on } H \label{E:Ifdef}
	\end{eqnarray}	
for all $f\in H_{B^*}$. Moreover, $B^*$ consists exactly of the elements $I(f)$ with $f \in H_{B^*}$: for any $\phi\in B^*$ the restriction $\phi|H$ is in $H^*$ (because of the relation \eqref{E:msnormwkr}) and hence $\phi=I(f)$ for a unique $f\in H_{B^*}$.

The fundamental result of Gross \cite{Gr} is that there is a unique Borel measure $\mu$ on $B$ such that for every $f\in H_{B^*}$ the  linear functional $I(f) $, viewed as a random variable on $(B, \mu)$, is Gaussian with mean $0$ and variance $|\!| f|\!|^2$; thus:
	\begin{equation}\label{E:muB}
	 \int_B e^{iI(f)}\,d\mu
	 =e^{-\frac{1}{2}|\!|f|\!|_H^2} 
	 \end{equation}
for all $f\in H_{B^*}\subset H$.  The mapping
	$$I: H_{B^*}\to L^2(B,\mu)$$
is a linear isometry. 

A triple $(H, B, \mu)$ where $H$ is a real separable Hilbert space, $B$ is the Banach space obtained by completing $H$ with respect to a measurable norm, and $\mu$ is the Gaussian measure in \eqref{E:muB}, is called an {\em abstract Wiener space}. The measure $\mu$ is known as {\em Wiener measure} on $B$.
	 
Suppose $g\in H$ is orthogonal to $H_{B^*}$; then for any $\phi\in B^*$ we have
	$$\phi(g)=\la f,g\ra=0$$
where $f$ is the element of $H_{B^*}$ for which  $I(f)=\phi$.  Since $\phi(g)=0$ for all $\phi\in B^*$ it follows that    $g=0$. Thus,
	\begin{equation}
	\hbox{\em $H_{B^*}$ is dense in $H$.}
	\end{equation}
Consequently $I$ extends uniquely to a linear isometry of $H$ into $L^2(B,\mu)$. We denote this  extension again by $I$:
	\begin{equation}\label{E:Ilinisom}
	I: H \to L^2(B,\mu).
	\end{equation}
It follows by taking $L^2$-limits that the random variable $I(f)$ is again Gaussian, satisfying (\ref{E:muB}), for every $f\in H$. 
	
It will be important for our purposes to emphasize again that {\em if  $f\in H$ is such that the  linear functional $\la f,\cdot\ra$ on $H$  is continuous with respect to the norm $|\cdot|$ then the random variable $I(f)$ is the unique  continuous linear functional on $B$ that agrees with $f$ on $H$.} If $\la f, \cdot \ra$ is {\em not} continuous on $H$ with respect to the norm $|\cdot|$ then $I(f)$ is obtained by $L^2(B,\mu)$-approximating with elements of $B^*$.

Conversely, one can show that any real separable Banach space $B$ equipped with a centered, non-degenerate Gaussian measure $\mu$, can be placed in the context of an abstract Wiener space.  The corresponding Hilbert space is the  \emph{Cameron-Martin space} of $(B, \mu)$, the subspace $H \subset B$ given by:
	\begin{equation*}
	H \stackrel{\rm def}{=} \{x \in B : \|x\| < \infty\},
	\end{equation*}
where for every $x \in B$:
	\begin{equation*}
	\|x\| \stackrel{\rm def}{=} \sup \{(x^*, x) : x^* \in B^*; \|x^*\|_{L^2(B, \mu)} \leq 1\}.
	\end{equation*}
The norm $\|\cdot\|$ defined above is a complete inner-product norm on $H$, it is stronger than the Banach norm $|\cdot|$ on $H$, and $H$ is dense in $B$. Moreover, the norm $|\cdot|$, restricted to $H$, is a measurable norm - for an interesting proof of this fact, due to Stroock, see section VIII of Driver's notes \cite{DriverNotes}. So $(H, B, \mu)$ is an abstract Wiener space and $H$ is uniquely determined by $B$ and $\mu$.

\subsubsection{Construction of the random variables \texorpdfstring{$\tilde{K}_p$}{ Kp}}\label{ss:cons}  As before, we assume given a model covariance structure
	   $$K:{\mathcal X}\times {\mathcal X}\to\mbr.$$ 
We have seen how this gives rise to a Hilbert space $H$, which is the closed linear span of the functions $K_p=K(p,\cdot)$, with $p$ running over ${\mathcal X}$.  Now let $|\cdot|$ be any measurable norm on $H$, and let  $B$ be the Banach space obtained by completion of $H$ with respect to $|\cdot|$. Let $\mu$ be  the centered Gaussian measure on $B$ as discussed above in  the context of \eqref{E:muB}. We set
	\begin{equation}\label{E:KpIKp}
	\tilde{K}_p\stackrel{\rm def}{=}I(K_p),
	\end{equation}
for all $p \in \mathcal{X}$; thus ${\tilde K}_p\in L^2(B,\mu)$. As a random variable  on $(B,\mu)$ the function ${\tilde K}_p$   is Gaussian, with mean $0$ and variance $\|K_p\|^2$, and the covariance structure is given by:
\begin{equation}\label{E:covKpqKqI}
\begin{split}
{\rm Cov}(\tilde{K}_p,\tilde{K}_q) &=\la \tilde{K}_p,\tilde{K}_q\ra_{L^2(\mu)}\\
&=\la K_p, K_q\ra_H\quad\hbox{(since $I$ is an isometry)}\\
&=K(p,q).
\end{split}
\end{equation}
Thus we have produced Gaussian random variables $\tilde{K}_p$, for each $p\in {\mathcal X}$,  on a Banach space, with a given covariance structure.

\subsubsection{Classical Wiener space}\label{ss:CWS} Let us look at the case of the classical Wiener space, as an example of the preceding structures. We take ${\mathcal X}=[0,T]$ for some positive real $T$, and 
\begin{equation}\label{E:KBM}
K^{\rm BM}(s,t)=\min\{s,t\}=\la 1_{[0,s]}, 1_{[0,t]}\ra_{L^2},
\end{equation}
for all $s, t\in [0,T]$, where the superscript refers to Brownian motion. Then $K^{\rm BM}_s(x)$ is $x$ for $x\in [0,s]$ and is constant at $s$ when $x\in [s, T]$. It follows then that $H_0$, the linear span of the functions $K^{\rm BM}_s$ for $s\in [0,T]$, is the set of all functions on $[0,T]$ with initial value $0$  and graph consisting of linear segments; in other words, $H_0$ consists of all functions $f$ on $[0,T]$ whose derivative exists and is locally constant outside a finite set of points. The inner product on $H_0$ is determined by observing that
\begin{equation}\label{E:KBMcov}
\la {K}^{\rm BM}_s, {K}^{\rm BM}_t\ra = K^{\rm BM} (s,t)=\min\{s,t\};
\end{equation}
 we can verify that this coincides with
	\begin{equation}
	\la f, g\ra=\int_0^Tf'(u)g'(u)\,du
	\end{equation}
for all $f, g\in H_0$.  Consequently,  $H$ is the Hilbert space consisting of  functions $f$ on $[0,T]$ that can be expressed as
\begin{equation}
f(t)=\int_0^tg(x)\,dx\qquad\hbox{for all $t\in [0,T]$,}
\end{equation}
for some $g\in L^2[0,T]$; equivalently, $H$ consists of all absolutely continuous  functions $f$ on $[0,T]$, with $f(0)=0$, for which the derivative $f'$ exists almost everywhere and is in $L^2[0,T]$. The sup norm $|\cdot|_{\sup}$ is a measurable norm on $H$ (see Gross \cite[Example 2]{Gr} or Kuo \cite{Ku1}).  The completion $B$ is then $C_0[0,T]$, the space of all continuous functions starting at $0$. If $f\in H_0$ then $I(f)$ is Gaussian of mean $0$ and variance $\|f\|^2$, relative to the Gaussian measure $\mu$ on  $B$.  We can check readily that $\tilde{K}^{\rm BM}_t=I(K^{\rm BM}_t)$ is Gaussian of mean $0$ and variance $t$ for all $t \in [0, T]$. The process $t\mapsto \tilde{K}^{\rm BM}_t$ yields standard Brownian motion, after a continuous version is chosen.

\subsection{The Gaussian Radon Transform}\label{S:backg}
 
The classical Radon transform $Rf$ of a function $f:\mbr^n\to \mbr$ associates to each hyperplane $P\subset\mbr^n$ the integral of $f$ over $P$; this is useful in scanning technologies and image reconstruction. The Gaussian Radon transform generalizes this to infinite dimensions and works with Gaussian measures (instead of Lebesgue measure, for which there is no useful infinite dimensional analog);  the motivation for studying this transform comes from the task of reconstructing a random variable from its conditional expectations. We have developed this transform in the setting of abstract Wiener spaces in our work \cite{HolSen2012} (earlier works, in other frameworks, include \cites{MS, BecSen2012, BecGR2010}). 
  
Let $H$ be a  real separable Hilbert space, and $B$ the Banach space  obtained by completing $H$ with respect to a measurable norm $|\cdot|$. Let $M_0$ be a closed subspace of $H$. Then 
\begin{equation}\label{E:Mpdef}
M_p=p+ {M_0}
\end{equation}
   is a   closed  affine subspace of $H$, for any point $p\in M_0^\perp$. In \cite[Theorem 2.1]{HolSen2012}   we have constructed a probability measure $\mu_{M_p}$ on $B$ uniquely specified by its Fourier transform 
 \begin{equation} \label{E:GRchar}
	\int_B e^{ix^*} \,d\mu_{M_p} = e^{i(p, x^*) - \frac{1}{2}\|P_{M_0} {x^*}\|^2} \text{, for all } x^* \in B^*\simeq H_{B^*}\subset H
	\end{equation}
where $P_{M_0}$ denotes orthogonal projection on $M_0$ in $H$. For our present purposes  we formulate the description of the measure $\mu_{M_p}$ in slightly different terms.   For every closed affine subspace $L\subset H$ there is a Borel probability measure $\mu_L$ on $B$ and  there is a linear mapping
$$I:H\to L^2(B,\mu_L)$$
such that for every $h\in H$, the  random variable $I(h)$  satisfies 
\begin{equation} \label{E:GRcharIh}
	\int_B e^{iI(h)} \,d\mu_{L} = e^{i\la f_L,h\ra  - \frac{1}{2}\|P {h}\|^2} \text{, for all } h\in H,
	\end{equation}
	where $f_L$ is the point on $L$ closest to the origin in $H$, and $P:H\to H$ is the orthogonal projection operator onto the closed subspace $-f_L+ L$. Moreover,  $\mu_{L}$ is concentrated on the closure of $L$ in $B$:
$$\mu_{L}(\overline{L}) = 1,$$
where $\overline{L}$ denotes the closure in $B$ of the affine subspace $L\subset H$.  (Note that the mapping $I$ depends on the subspace $L$.  For more details about $I_{L}$, see Corollary 3.1.1 and observation (v) following Theorem 2.1 in \cite{Hol2013}.)

From the condition (\ref{E:GRcharIh}), holding for all $h\in H$,  we see that    $I(h)$ is Gaussian with mean $\la h, f_L\ra$ and variance $\|P {h}\|^2$.

The {\em Gaussian Radon transform} $Gf$ for a Borel function $f$ on $B$ is a function defined on the set of all closed affine subspaces $L$ in $H$ by:
	\begin{equation}\label{E:GfP}
	Gf(L) \stackrel{\text{def}}{=} \int_B f \,d\mu_L.
	\end{equation}
	(Of course, $Gf$  is  defined if the integral is finite for all such $L$.)  
	
 The value $Gf(L)$ corresponds to the conditional expectation of $f$, the conditioning being given by the closed affine subspace $L$. For a precise formulation of this and proof for $L$ being of finite codimension see the disintegration formula given in \cite[Theorem 3.1]{Hol2013}. The following is an immediate consequence of \cite[Corollary 3.2]{Hol2013} and will play a role in Section \ref{S:gml}:

\begin{prop}\label{P:CondExpGf}
Let $(H, B, \mu)$ be an abstract Wiener space and linearly independent elements $h_1, h_2, \ldots, h_n$ of $H$. Let $f$ be a  Borel function  on $B$,  square-integrable with respect to $\mu$, and  let
	\begin{equation*}
	F : \mathbb{R}^n \rightarrow \mathbb{R}:(y_1,\ldots, y_n)\mapsto  F(y_1, \ldots, y_n) = Gf\left( \bigcap_{j=1}^n \left[ \left<h_j, \cdot\right> = y_j \right] \right).
	\end{equation*}
Then $F\bigl(I(h_1), I(h_2), \ldots, I(h_n)\bigr)$ is a version of  the conditional expectation 
$$\mathbb{E}\left[f| \sigma\left(I(h_1), I(h_2), \ldots, I(h_n)\right)\right].$$ 
\end{prop}

\section{Machine Learning using the Gaussian Radon Transform}\label{S:gml}

Consider again the learning problem outlined in the introduction: suppose $\mathcal{X}$ is a separable topological space and $H$ is the RKHS over $\mathcal{X}$ with reproducing kernel $K:\mathcal{X}\times\mathcal{X} \rightarrow\mathbb{R}$. Let:
	\begin{equation}
	D = \{ (p_1, y_1), (p_2, y_2), \ldots, (p_n, y_n)\} \subset \mathcal{X} \times \mathbb{R}
	\end{equation}
be our training data and suppose we want to predict the value $y$ at a future value $p\in\mathcal{X}$. Note that $H$ is a real separable Hilbert space so, as outlined in Section \ref{ss:gmbs}, we may complete $H$ with respect to a measurable norm $|\cdot|$ to obtain the Banach space $B$ and Gaussian measure $\mu$. Moreover, recall from Section \ref{ss:cons} that we constructed for every $p\in \mathcal{X}$ the random variable 
	\begin{equation}
	\tilde{K}_p = I(K_p) \in L^2(B, \mu),
	\end{equation}
where $I: H \rightarrow L^2(B, \mu)$ is the isometry in \eqref{E:Ilinisom}, and
	\begin{equation}
	\text{Cov}(\tilde{K}_{p}, \tilde{K}_{q}) = \left<K_{p}, K_{q}\right> = K(p, q)
	\end{equation}
for all $p,q\in\mathcal{X}$. Therefore $\{\tilde{K}_p: p\in\mathcal{X}\}$ is a centered Gaussian process with covariance $K$.

Since for every $f\in H$ the value $f(p)$ is given by $\left<K_p, f\right>$, we work with $\tilde{K}_p(f)$ as our desired random-variable prediction. The first  guess  would therefore be to use:
	\begin{equation}\label{E:FirstInstinct}
	\hat{y} = \mathbb{E}\left[ \tilde{K}_p | \tilde{K}_{p_1} = y_1, \ldots, \tilde{K}_{p_n}=y_n \right]
	\end{equation}
as our prediction of the output $y$ corresponding to the input $p$. 

Using Lemma \ref{L:condExpGauss} the prediction in \eqref{E:FirstInstinct} becomes:
	\begin{eqnarray*}
	\hat{y} &=& \left[ K(p_1, p), \ldots ,K(p_n, p) \right](K_D)^{-1}y\\
	&=& \sum_{j=1}^n b_j K_{p_j}(p)
	\end{eqnarray*}
where $b = (K_D)^{-1}y$ and $K_D$ is, as in the introduction, the matrix with entries $K(p_i, p_j)$. As we shall see in Theorem \ref{T:splinemin}, the function $\sum_{j=1}^n b_jK_{p_j}$ is the element $\hat{f}_0 \in H$ of minimal norm such that $\hat{f}_0(p_i) = y_i$ for all $1 \leq i \leq n$. Therefore \eqref{E:FirstInstinct} is the solution in the traditional spline setting, and does not take into account the regularization parameter $\lambda$. Combining this with Proposition \ref{P:CondExpGf}, we have:

\begin{theorem} \label{T:SplineGf}
Let $(H, B, \mu)$ be an abstract Wiener space, where $H$ be the real RKHS over a separable topological space $\mcx$ with reproducing kernel $K$. Given:
	$$D = \{ (p_1, y_1), \ldots, (p_n, y_n)\} \subset \mcx \times \mbr,$$
such that $K_{p_1}, \ldots, K_{p_n}$ are linearly independent, the element $\hat{f}_0 \in H$ of minimal norm that satisfies $\hat{f}_0(p_j) = y_j$ for all $1 \leq j \leq n$ is given by:
	\begin{equation} \label{E:SplineGf}
	\hat{f}_0(p) = G\tilde{K}_p\left( \bigcap_{j=1}^n \left[ \la K_{p_j}, \cdot\ra = y_j \right] \right),
	\end{equation}
where $G\tilde{K}_p$ is the Gaussian Radon transform of $\tilde{K}_p$ for every $p \in \mcx$.
\end{theorem}

The next theorem shows that the ridge regression solution can also be obtained in terms of the Gaussian Radon transform, by taking the Gaussian process approach outlined in the introduction. 

\begin{theorem}\label{T:fhatascp}
Let $H$ be the RKHS over a separable topological space $\mathcal{X}$ with real-valued reproducing kernel $K$ and $B$ be the completion of $H$ with respect to a measurable norm $|\cdot|$ with Wiener measure $\mu$. Let
	\begin{equation} \label{E:trd}
	D = \{ (p_1, y_1), (p_2, y_2), \ldots, (p_n, y_n)\} \subset \mathcal{X} \times \mathbb{R}
	\end{equation}
be fixed and $p\in X$. Let $\{e_1, e_2, \ldots, e_n\} \subset H$ be an orthonormal set    such that:
	\begin{equation}\label{E:es}
	\{e_1, \ldots, e_n\} \subset [\text{span}\{K_{p_1}, \ldots, K_{p_n}, K_p\}]^{\perp}
	\end{equation}
 and for every $1\leq j \leq n$ let $\tilde{e}_j = I(e_j)$ where $I:H\rightarrow L^2(B,\mu)$ is the isometry in \eqref{E:Ilinisom}. Then for any $\lambda > 0$:
	\begin{equation}\label{E:Thm1}
	\mathbb{E}\left[ \tilde{K}_p | \tilde{K}_{p_1} + \sqrt{\lambda}\tilde{e}_1 = y_1, \ldots, \tilde{K}_{p_n} + \sqrt{\lambda}\tilde{e}_n = y_n\right] = \hat{f}_{\lambda}(p)
	\end{equation}
where $\hat{f}_{\lambda}$ is the solution to the ridge regression problem in \eqref{E:RegRegr} with regularization parameter $\lambda$. Consequently:
	\begin{equation} \label{E:Thm2}
	\hat{f}_{\lambda}(p) = G\tilde{K}_p \left( \bigcap_{j=1}^n \left[ \left<K_{p_j} + \sqrt{\lambda} e_j, \cdot\right> = y_j \right] \right)
	\end{equation}
where $G\tilde{K}_p$ is the Gaussian Radon transform of $\tilde{K}_p$.
\end{theorem}

Note that a completely precise statement of the relation (\ref{E:Thm1}) is as follows. Let us for the moment write the quantity ${\hat f}_{\lambda}(p)$, involving the vector $y$, as
$${\hat f}_{\lambda}(p;y).$$
Then
\begin{equation}
{\hat f}_{\lambda}\left(p; \Bigl(\tilde{K}_{p_1} + \sqrt{\lambda}\tilde{e}_1, \ldots, \tilde{K}_{p_n} + \sqrt{\lambda}\tilde{e}_n\Bigr)\right)
\end{equation}
is a version of the conditional expectation
\begin{equation}\label{E:Thm12}
	\mathbb{E}\left[ \tilde{K}_p\, |\, \sigma\left(\tilde{K}_{p_1} + \sqrt{\lambda}\tilde{e}_1, \ldots, \tilde{K}_{p_n} + \sqrt{\lambda}\tilde{e}_n \right)\right]. 
	\end{equation}
	
\begin{proof}
 By our assumption in \eqref{E:es}:
	\begin{eqnarray*}
	\text{Cov}\left( \tilde{K}_{p_i} + \sqrt{\lambda}\tilde{e}_j, \tilde{K}_{p_j} + \sqrt{\lambda}\tilde{e}_j \right) &=& \left< K_{p_i} + \sqrt{\lambda}e_i, K_{p_j} +\sqrt{\lambda}e_j \right>\\
	&=& K(p_i, p_j) + \lambda \delta_{i,j}
	\end{eqnarray*}
for all $1\leq i, j \leq n$. We note that this covariance matrix is invertible (see the argument preceding equation (\ref{E:f0Tb}) below).  Moreover,
	\begin{equation*}
	\text{Cov}\left(\tilde{K}_p, \tilde{K}_{p_j} + \sqrt{\lambda}\tilde{e}_j\right) = K(p_j, p)
	\end{equation*}
for all $1\leq j \leq n$. Then by Lemma \ref{L:condExpGauss}:
	\begin{eqnarray*}
	&&\mathbb{E} \left[ \tilde{K}_p | \tilde{K}_{p_1} + \sqrt{\lambda}\tilde{e}_1 = y_1, \ldots, \tilde{K}_{p_n} + \sqrt{\lambda}\tilde{e}_n = y_n \right] \\
	&&= \left[ K(p_1, p), \ldots ,K(p_n, p) \right](K_D + \lambda I_n)^{-1}y\\
	&&= \sum_{j=1}^n c_j K_{p_j}(p) \text{ where } c = (K_D + \lambda I_n)^{-1}y\\
	&&= \hat{f}_{\lambda}(p)
	\end{eqnarray*}
Finally, \eqref{E:Thm2} follows directly from Proposition \ref{P:CondExpGf}.
\end{proof}

The interpretation of the predicted value in terms of the Gaussian Radon transform allows for quite a broad class of functions that can be considered for prediction. As a simple example, consider the task of predicting the maximum value of an unknown function over a future period using knowledge from the training data. The predicted value would be
\begin{equation}\label{E:GMpredictmax}
GF(L)
\end{equation}
where $L$ is the closed affine subspace of the RKHS reflecting the training data, and $F$ is, for example, a function of the form
\begin{equation}
F(x)=\sup_{p\in S}{\tilde K}_p(x) 
\end{equation}
for some given set $S$ of `future dates'. We note that  the prediction (\ref{E:GMpredictmax}) is, in general,  {\em different} from 
$$\sup_{p\in S}G {\tilde K}_p(L),  
$$
where  $G {\tilde K}_p(L)$ is the SVM prediction as in (\ref{E:Thm2});
in other words, the prediction (\ref{E:GMpredictmax}) is not the same as simply taking the supremum over the  predictions given by the SVM minimizer.
We note also that in this type of problem the Hilbert space, being a function space, is  necessarily infinite-dimensional.

\subsection{An approach using direct sums of abstract Wiener spaces} \label{sS:DirSumAWS}

One disadvantage of Theorem \ref{T:fhatascp} is that the choice of $\{e_1, \ldots, e_n\}$ could change with every training set and every future input $p \in \mathcal{X}$ - specifically, given the training data \eqref{E:trd}, one must choose an orthonormal set $\{e_1, \ldots, e_n\} \subset H$ that is not only orthogonal to each $K_{p_i}$, but also to $K_p$, for every future input $p$ whose outcome we would like to predict. Clearly a set $\{e_1, \ldots, e_n\}$ that would ``universally'' work cannot be found in $H$, because span$\{K_p : p \in \mathcal{X}\}$ is dense in $H$. We would therefore like to attach to $H$ another copy of $H$ which would function as a ``repository'' of errors - so the two spaces would need to be in a sense orthogonal to each other. This is precisely the idea behind direct sums of Hilbert spaces: if $(H_1, \|\cdot\|_1)$ and $(H_2, \|\cdot\|_2)$ are Hilbert spaces, their orthogonal direct sum:
	\begin{equation*}
	H_1 \oplus H_2 \defeq \{ (h_1, h_2) : h_1 \in H_1, h_2 \in H_2 \}
	\end{equation*}
is a Hilbert space with inner-product:
	\begin{equation*}
	\la (h_1, h_2), (g_1, g_2)\ra \defeq \la h_1, h_2\ra_1 + \la h_2, g_2\ra_2,
	\end{equation*}
for all $h_1, g_1 \in H_1$ and $h_2, g_2 \in H_2$.

\begin{prop} \label{P:AWSDirectSum}
Let $(H_1, B_1, \mu_1)$ and $(H_2, B_2, \mu_2)$ be abstract Wiener spaces, where $(H_1, \|\cdot\|_1)$ and $(H_2, \|\cdot\|_2)$ are real separable infinite-dimensional Hillbert spaces and $B_1$, $B_2$ are the Banach spaces obtained by completing $H_1$, $H_2$ with respect to measurable norms $|\cdot|_1$, $|\cdot|_2$, respectively. Then:
	\begin{equation*}
	(H_1 \oplus H_2, B_1 \oplus B_2, \mu_1 \times \mu_2)
	\end{equation*}
is an abstract Wiener space, where $B_1 \oplus B_2 = \{(x_1, x_2) : x_1 \in B_1, x_2 \in B_2\}$ is a separable Banach space with the norm $|(x_1, x_2)| = |x_1|_1 + |x_2|_2$, for all $x_1 \in B_1$, $x_2 \in B_2$.
\end{prop}

It then easily follows that if $I : H_1 \oplus H_2 \rightarrow (B_1 \oplus B_2, \mu_1 \times \mu_2)$ is the isometry described in \eqref{E:Ilinisom}, then:
	\begin{equation*}
	I(h_1, h_2)(x_1, x_2) = (I_1h_1)(x_1) + (I_2h_2)(x_2),
	\end{equation*}
where $I_k : H_k \rightarrow L^2(B_k, \mu_k)$ is the corresponding isometry for $k = 1,2$. 

Going back to the ridge regression problem, let now $(H, B, \mu)$ be an abstract Wiener space, where $H$ is the RKHS over a separable topological space $\mathcal{X}$ with reproducing kernel $K$ and $D = \{(p_1, y_1), \ldots, (p_n, y_n)\} \subset \mathcal{X} \times \mathbb{R}$ be our training data. Recall that we would like an ``orthogonal repository'' for the errors - so let $(H', B', \mu')$ be an abstract Wiener space, where $H'$ is a real separable infinite-dimensional Hilbert space. For every $p \in \mathcal{X}$ let $K_p \defeq K(p, \cdot) \in H$ and:
	\begin{equation*}
	\tilde{K}_p \defeq I(K_p, 0) \in L^2(B \oplus B', \mu \times \mu'),
	\end{equation*}
where $I_{\oplus}$ is the isometry in \eqref{E:Ilinisom}. As previously noted:
	\begin{equation*}
	\tilde{K}_p(x, x') = (IK_p)(x) \text{, for almost all } x \in B, x' \in B',
	\end{equation*}
where $I : H \rightarrow L^2(B, \mu)$. 

Let $\lambda > 0$ and $\{e_j\}_{j \geq 1}$ be an orthonormal basis for $H'$ and for every positive integer $j$ let:
	\begin{equation*}
	\tilde{e}_j \defeq I(0, e_j); \tilde{e}_j(x, x') = (I'e_j)(x')
	\end{equation*}
where $I' : H' \rightarrow L^2(B', \mu')$. Then:
	\begin{equation*}
	I(K_p, \sqrt{\lambda}e_j) = \tilde{K}_p + \sqrt{\lambda}\tilde{e}_j
	\end{equation*}
for all $p \in \mathcal{X}$ and $j \in \mathbb{N}$. Then by Lemma \ref{L:condExpGauss} and Proposition \ref{P:CondExpGf}, for any $p \in \mathcal{X}$:
	\begin{eqnarray*}
	\hat{f}_{\lambda} &=& \mathbb{E} [\tilde{K}_p | \tilde{K}_{p_j} + \sqrt{\lambda}\tilde{e}_j = y_j, 1 \leq j \leq n]\\
	&=& G\tilde{K}_p\left( \bigcap_{j=1}^n [\la (K_{p_j}, \sqrt{\lambda}e_j), \cdot\ra = y_j] \right),
	\end{eqnarray*}
where $\hat{f}_{\lambda}$ is the ridge regression solution and both the conditional expectation and the Gaussian Radon transform above are with respect to $(B \oplus B', \mu\times\mu')$. In this approach, for any number $n$ of training points and any future input $p \in \mathcal{X}$ we may simply work with $\{e_1, \ldots, e_n\}\subset H'$, and we no longer have to choose this orthonormal set for every training set and every new input.

We return to the proof of  Proposition \ref{P:AWSDirectSum}.

\begin{proof}
Every continuous linear functional on $B_1 \oplus B_2$ is of the form $x^* = (x_1^*, x_2^*)$ for some $x_1^* \in B_1^*$ and $x_2^* \in B_2^*$, where:
	\begin{equation*}
	((x_1, x_2), (x_1^*, x_2^*)) = (x_1, x_1^*) + (x_2, x_2^*) \text{, for all } x_1\in B_1, x_2 \in B_2.
	\end{equation*}
Since $|\cdot|_1$ and $|\cdot|_2$ are weaker than $\|\cdot\|_1$ and $\|\cdot\|_2$ on $H_1$ and $H_2$, respectively, there is $c > 0$ such that $|h_k|_k \leq c\|h_k\|_k$ for all $h_k \in H_k$ and all $k = 1, 2$. Then:
	\begin{eqnarray*}
	|(h_1, h_2)|^2 &=& (|h_1|_1 + |h_2|_2)^2\\
	&\leq& 2c^2 (\|h_1\|_1^2 + \|h_2\|_2^2) = 2c^2 \|(h_1, h_2)\|^2,
	\end{eqnarray*}
which shows that $|(\cdot, \cdot)|$ is a weaker norm than $\|(\cdot, \cdot)\|$ on $H_1 \oplus H_2$. Consequently, to every $x^* = (x_1^*, x_2^*) \in (B_1 \oplus B_2)^*$ we may associate a unique $h_{x^*} \in H_1 \oplus H_2$ such that $((h_1, h_2), x^*) = \la (h_1, h_2), h_{x^*}\ra$ for all $h_1 \in H_1$, $h_2 \in H_2$. As can be easily seen, this element is exactly $h_{x^*} = (h_{x_1^*}, h_{x_2^*})$, where $h_{x_k^*} \in (H_k)_{B_k^*}$ for $k = 1, 2$. 

The characteristic functional of the measure $\mu_1 \times \mu_2$ on $B_1 \oplus B_2$ is then:
	\begin{equation*}
	\int_{B_1 \oplus B_2} e^{i(x_1^*, x_2^*)}\,d(\mu_1\times\mu_2) = e^{-\frac{1}{2}\|(h_{x_1^*}, h_{x_2^*})\|^2}
	\end{equation*}
for all $(x_1^*, x_2^*)$. Therefore $\mu_1\times\mu_2$ is a centered non-degenerate Gaussian measure with covariance operator:
	\begin{equation*}
	R_{\mu_1\times\mu_2}\left( (x_1^*, x_2^*), (y_1^*, y_2^*) \right) = \la (h_{x_1^*}, h_{x_2^*}), (h_{y_1^*}, h_{y_2^*}) \ra,
	\end{equation*}
for all $x_k^*$, $y_k^*$. Define for $(x_1, x_2)\in B_1 \oplus B_2$:
	\begin{equation}\label{E:normprimeB12}
	\|(x_1, x_2)\|' \defeq \sup_{\stackrel{x_1^* \in B_1^*, x_2^* \in B_2^*}{\|h_{x_1^*}\|_1^2 + \|h_{x_2^*}\|_2^2\leq 1}}\Bigl(|(x_1, x_1^*) + (x_2, x_2^*)|\Bigr)
	\end{equation}
and the Cameron-Martin space $H$ of $(B_1 \oplus B_2, \mu_1 \times \mu_2)$:
	\begin{equation}\label{E:defHB12}
	H \defeq \{ (x_1, x_2) \in B_1 \oplus B_2 : \|(x_1, x_2)\|' < \infty\}.
	\end{equation}
For every $(h_1, h_2) \in H_1 \oplus H_2$ and $(x_1^*, x_2^*) \in (B_1 \oplus B_2)^*$, 
	\begin{eqnarray*}
	|(h_1, x_1^*) + (h_2, x_2^*)| &\leq& |\la h_1, h_{x_1^*}\ra_1| + |\la h_2, h_{x_2^*}\ra_2|\\
	&\leq& \|h_1\|_1 \|h_{x_1^*}\|_1 + \|h_2\|_2\|h_{x_2^*}\|_2\\
	&\leq& \sqrt{\|h_1\|_1^2 + \|h_2\|_2^2} \sqrt{\|h_{x_1^*}\|_1^2 + \|h_{x_2^*}\|_2^*},
	\end{eqnarray*}
so:
	\begin{equation*}
	\|(h_1, h_2)\|' \leq \|(h_1, h_2)\| < \infty,
	\end{equation*}
which shows that $H_1 \oplus H_2 \subset H$. 

Conversely, suppose $(x_1, x_2) \in H$. Then by letting $x_2^* = 0$:
	\begin{equation*}
	\sup_{\stackrel{x_1^* \in B_1^*}{\|h_{x_1^*}\|_1\leq 1}} |(x_1, x_1^*)| \leq \|(x_1, x_2)\|' < \infty,
	\end{equation*}
so $x_1$ is in the Cameron-Martin space of $(B_1, \mu_1)$ - which is just $H_1$. Similarly, $x_2 \in H_2$, proving that $H = H_1 \oplus H_2$ as sets. To see that the norms also correspond, note that for any $y_1^* \in B_1^*$ and $y_2^* \in B_2^*$, not both $0$:
	\begin{eqnarray*}
	\|(h_{y_1^*}, h_{y_2^*})\|' &=& \sup_{\stackrel{x_1^* \in B_1^*, x_2^* \in B_2^*}{\|h_{x_1^*}\|_1^2 + \|h_{x_2^*}\|_2^2\leq 1}} \Bigl(|(h_{y_1^*}, x_1^*) + (h_{y_2^*}, x_2^*)|\Bigr) \\
	&\geq& \frac{|\la h_{y_1^*}, h_{y_1^*}\ra_1 + \la h_{y_2^*}, h_{y_2^*}\ra_2|}{\sqrt{\|h_{y_1^*}\|_1^2 + \|h_{y_2^*}\|_2^2}}\\
	&=& \|(h_{y_1^*}, h_{y_2^*})\|.
	\end{eqnarray*}
So:
	\begin{equation} \label{E:DirSumAWSP1}
	\|h_{x^*}\|' \geq \|h_{x^*}\| \text{, for all } x^* \in (B_1 \oplus B_2)^*.
	\end{equation}
Since $\{h_{x^*}: x^* \in (B_1 \oplus B_2)^*\}$ is dense in both $H$ and $H_1 \oplus H_2$, \eqref{E:DirSumAWSP1} proves that $\|(h_1, h_2)\|' \geq \|(h_1, h_2)\|$ for all $h_1, h_2$, so $H$ and $H_1 \oplus H_2$ are the same as Hilbert spaces. Thus $H_1 \oplus H_2$ is the Cameron-Martin space of $(B_1 \oplus B_2, \mu_1 \times \mu_2)$, which concludes our proof.
\end{proof}

\section{Realizing \texorpdfstring{$B$}{B} as a space of functions}\label{S:Basfunct}

In this section we present some results of a somewhat technical nature to address the question as to whether the elements of the Banach space $B$ can be viewed as functions. 

\subsection{Continuity of \texorpdfstring{$\tilde{K}_p$}{Kp}}\label{ss:contKp} A general measurable norm on $H$ does not `know' about the kernel function $K$ and hence there seems to be no reason why the functionals $\la K_p, \cdot \ra$ on $H$ would be continuous with respect to the norm $|\cdot|$. To remedy this situation we prove   that   there exist measurable norms on $H$ relative to which the  functionals $\la K_p, \cdot\ra$ are continuous for $p$ running along a dense sequence of points in ${\mathcal X}$:

\begin{prop}\label{P:Kpcontmeas} 
Let $H$ be the reproducing kernel Hilbert space associated to a  continuous  kernel function $K:{\mathcal X}\times {\mathcal X}\to\mbr$, where ${\mathcal X}$ is a separable topological space.  Let $D$  be a countable dense subset of ${\mathcal X}$. Then there is a measurable norm $|\cdot|$ on $H$ with respect to which  $\la K_p, \cdot\ra$ is  a continuous linear functional for every $p\in D$.
\end{prop}
 
\noindent\underline{Proof}.  As noted in the context of (\ref{E:Phipq}), the feature map ${\mathcal X}\to H: q\mapsto K_q$ is continuous. So if $v\in H$ then the map $q\mapsto\la v, K_q\ra$ is continuous and so if $p\in{\mathcal X}$ is a point for which $\la v, K_p\ra\neq 0$ then there is a neighborhood $U$ of $p$ such that $\la v, K_q\ra\neq 0$ for all $q\in U$. Since $D$ is dense in ${\mathcal X}$  we conclude that there is a point $p'\in D$ for which $\la v, K_{p'}\ra\neq 0$.  Turning this argument into its contrapositive, we see that a vector orthogonal to  $K_{p'}$ for every $p'\in D$ is orthogonal to $K_p$ for all $p\in {\mathcal X}$ and hence is $0$  because the span of $\{K_p:p\in{\mathcal X}\}$ is dense in $H$.   Thus $\{K_{p_1}, K_{p_2},\ldots\}$ spans a dense subspace of $H$, where $D=\{p_1, p_2,\ldots\}$. By the Gram-Schmidt process we obtain an orthonormal basis $e_1, e_2, \ldots$ of $H$ such that $K_{p_n}$ is contained in the span of $e_1,\ldots, e_n$, for every $n$.   Now consider the  bounded linear operator
$A:H\to H$ specified  by requiring that $Ae_n=\frac{1}{n}e_n$ for all $n$; this is Hilbert-Schmidt because $\sum_{n}\|Ae_n\|^2<\infty$ and is clearly injective. Hence, by   Property (3) discussed in the context of   \eqref{E:HilbSchmidt},
\begin{equation}\label{E:deffnormAf}
|f|\stackrel{\rm def}{=}\|Af\|=\left[\sum_n\frac{1}{n^2}\la e_n,f\ra^2\right]^{1/2} \quad\hbox{for all $f\in H$}
\end{equation}
specifies a measurable norm on $H$. Then
$$|\la e_n, f\ra|\leq n |f|,$$
from which we see that the linear functional $\la e_n,\cdot\ra$ on $H$ is continuous with respect to the norm $|\cdot|$. Hence, by definition of the linear isometry $I:H\to L^2(B,\mu)$ given in (\ref{E:Ifdef}),  $I(e_n)$ is the element in $B^*$ that agrees with $\la e_n,\cdot\ra$ on $H$. In particular each $I(e_n)$ is continuous and hence $\tilde{K}_{p}=I(K_p)$ is a continuous linear functional on $B$ for every $p\in D$. \fbox{QED}


The measurable norm $|\cdot |$ we have constructed in the preceding proof arises from a (new) inner-product on $H$. However, given any other measurable norm $ |\cdot |_0$ on $H$ the sum 
$$|\cdot|'\stackrel{\rm def}{=} |\cdot |+ |\cdot  |_0$$
is also a measurable norm (not necessarily arising from an inner-product) and  the linear functional $\la K_p, \cdot\ra:H\to\mbr $ is continuous with respect to the norm $|\cdot|'$ for every $p\in D$.

\subsection{Elements of \texorpdfstring{$B$}{B} as functions} \label{ss:Basfunct} If a  Banach space $B$  is obtained by completing a Hilbert space $H$ of functions, the elements of $B$ need not   consist of functions. However, when $H$ is a reproducing kernel Hilbert space as we have been discussing and under reasonable conditions on the reproducing  kernel function $K$ it is true that elements of $B$ can `almost' be thought of as functions on ${\mathcal X}$.  For this we first develop a lemma:

    \begin{lemma}\label{L:BB0} Suppose $H$ is a separable  real Hilbert space and $B$ the Banach space obtained by completing $H$ with respect to a measurable norm $|\cdot|$. Let $B_0$ be a closed subspace of $B$ that is transverse to $H$ in the sense that $H\cap B_0=\{0\}$, and let
    \begin{equation}\label{E:Bstar}
    B_{1}=B/B_0=\{b+B_0\,:\,b\in B\}
    \end{equation}
    be the quotient Banach space, with the standard quotient norm, given by
    \begin{equation}\label{E:normstar}
    |b+B_0|_{1}\stackrel{\rm def}{=}\inf\{|x|\,:\, x\in b+B_0\}.
    \end{equation}
    Then the mapping
  \begin{equation}\label{E:istar}
  \iota_1:H\to B/B_0:h\mapsto  \iota(h)+B_0,
  \end{equation}    where $\iota:H\to B$ is the inclusion,  is a continuous linear injective map, and 
    \begin{equation}\label{E:normstar1}
    |h|_{1}\stackrel{\rm def}{=}|\iota_1(h)|_1 \quad\hbox{for all $h\in H$,}
    \end{equation} 
    specifies a measurable norm on $H$.  The image of $H$ under $\iota_1$ is a dense subspace of $B_1$. \end{lemma}
    \noindent\underline{Proof}. Let us first note that by definition of the quotient norm
    $$|b+B_0|_1\leq |b|\qquad\hbox{for all $b\in B$.}$$
    Hence
    $$|h|_1\leq   |h|\qquad\hbox{ for all $h\in H\subset  B$.}$$
    Let $\epsilon>0$. Then since $|\cdot|$ is a measurable norm on $H$ there is a finite-dimensional subspace $F_\epsilon\subset H$ such that if $F$ is any finite-dimensional subspace of $H$   orthogonal to $F_\epsilon$ then, as noted back in \eqref{E:hepsilon}, 
    $$\gamma_F\{ h\in F\,:\, |h|>\epsilon\}<\epsilon,$$
    where $\gamma_F$ is standard Gaussian measure on $F$.
    Then
    \begin{equation}\begin{split}
    \gamma_F\{ h\in F\,:\, |h|_1>\epsilon\} &\leq   \gamma_F\{ h\in F\,:\, |h| >\epsilon\}\\
    &<\epsilon,
    \end{split}
    \end{equation}
    where the first inequality holds because whenever $|h|_1>\epsilon$ we also have $|h|\geq |h|_1>\epsilon$. Thus, $|\cdot|_1$ is a measurable norm on $H$.   The image $\iota_1(H)$ is the same as the projection of the dense subspace $H\subset B$ onto the quotient space $B_1= B/B_0$ and hence  this image is dense in $B_1$ (an open set in the complement of $\iota_1(H)$ would have inverse image in $B$ that is in the complement of $H$, and would have to be empty because $H$ is dense in $B$).
     \fbox{QED}
    
    We can now establish the identification of $B$ as a function space.
    
 \begin{prop}\label{P:BasFunct}  Let $K:{\mathcal X}\times {\mathcal X}\to\mbr$ be a continuous function, symmetric and non-negative definite, where $\mathcal{X}$ is a separable topological space, and $D$ a countable dense subset of ${\mathcal X}$. Let $H$ be the corresponding reproducing kernel Hilbert space. Then there is a measurable norm $|\cdot|_1$ on $H$ such that   the Banach space $B_1$ obtained by completing $H$ with respect to $|\cdot|_1$ can be realized as a space of functions on the set $D$. \end{prop}

 \noindent\underline{Proof}. Let $B$ be the completion of $H$ with respect to a measurable norm $|\cdot|$ of the type given in Proposition \ref{P:Kpcontmeas}. Thus $\la K_p, \cdot\ra:H\to\mbr$ is continuous  with respect to $|\cdot|$  when $p\in D$; let
 $$\tilde{K}_p:B\to\mbr$$
 be the continuous linear extension of $K_p$ to the Banach space $B$, for $p\in D$. 
   Now let
 \begin{equation}\label{E:defB0}
 B_0\stackrel{\rm def}{=}\cap_{p\in D}\ker \tilde{K}_p,
 \end{equation}
  a closed subspace of $B$.  We observe that $B_0$ is transverse to $H$; for if $x\in B_0$ is in $H$ then $\la K_p,x\ra=0$ for all $p\in D$ and so $x=0$ since $\{K_p:p\in D\}$ spans a dense subspace of $H$  as noted in Theorem \ref{T: covHilb} and the remark following it.   Then by   Lemma \ref{L:BB0}, $B_{1}=B/B_0$ is a Banach space that is a completion of  $H$ in the sense that $\iota_1:H\to B_1:h\mapsto h+B_0$ is continuous linear with dense image and $|h|_1\stackrel{\rm def}{=}|\iota_1(h)|_1$, for $h\in H$,  specifies a measurable norm on $H$. Let $K^1_p$ be the  linear functional on $B_1$ induced by $\tilde{K}_p$:
  \begin{equation}\label{E:defKpB1}
  K^1_p(b+B_0)\stackrel{\rm def}{=}\tilde{K}_p(b)\qquad\hbox{for all $b\in B$,}
  \end{equation} 
  which is well-defined because the linear functional $\tilde{K}_p$ is $0$ on $B_0$.  We note that
  $$K^1_p\bigl(\iota_1(h)\bigr)= K^1_p(h+B_0)=\tilde{K}_p(h)=\la K_p, h\ra $$
  for all $h\in H$, and so $K^1_p$ is the continuous linear `extension' of $\la K_p, \cdot\ra $ to $B_1$ through  $\iota_1:H\to B_1$, viewed as an `inclusion' map.  
  
  Now to each $b_1 \in B_1$ associate the function $f_{b_1}$ on $D$ given by:
		$$f_{b_1}: D \rightarrow \mathbb{R}: p \mapsto K^1_p(b_1).$$
		We will shows that the mapping
		$$b_1\mapsto f_{b_1}$$
		is injective; thus it realizes $B_1$ as a set of functions on the set $D$. To this end, 
	suppose that
  $$f_{b+B_0}=f_{c+B_0}$$
  for some $b, c\in B$. This means  
    $$K^1_p(b+B_0)=K^1_p(c+B_0)\qquad\hbox{for all $p\in D$,}$$
 and so $\tilde{K}_p(b) = \tilde{K}_p(c)$ for all $p \in D$. Then 
    $$b-c\in \ker \tilde{K}_p\qquad\hbox{for all $p\in D$.}$$
  Thus $b-c\in B_0$ and so $b+B_0=c+B_0$. Thus we have shown that $b_1\mapsto f_{b_1}$ is injective.  \fbox{QED}

 We have defined the function $f_b$ on the set $D\subset {\mathcal X}$, with notation and hypotheses as above. Now taking a general point $p\in {\mathcal X}$ and a sequence of points $p_n\in D$ converging to $p$ the function
 $\tilde{K}_p$ on $B$  is the $L^2(\mu)$-limit  the sequences of functions $\tilde{K}_{p_n}$. Thus we can define
 $ f_b(p)=\tilde{K}_p(b)$, with the understanding that for a given $p\in {\mathcal X}$, the value $f_b(p)$ is $\mu$-almost-surely defined in terms of its dependence on $b$. In the theory of Gaussian random fields one has conditions on the covariance function $K$ that ensure that $D\times B\to\mbr: (p,b)\mapsto \tilde{K}_p(b)$ is continuous in $p$ for $\mu$-a.e. $b\in B$, and in this case the function $f_b$ extends uniquely to a continuous function on ${\mathcal X}$, for $\mu$-almost-every $b\in B$.

\appendix

\section{A geometric formulation}\label{S:geomform}

In this section we present a geometric view of    the relationship between the Gaussian Radon transform and the representer theorem used in support vector machine theory; thus, this will be a geometric interpretation of Theorem \ref{T:fhatascp}. 

Given a reproducing kernel Hilbert space $H$ of functions defined on a set ${\mathcal X}$ with reproducing kernel $K: \mathcal{X}\times\mathcal{X}\rightarrow\mathbb{R}$, we wish to find a function $\hat{f}_{\lambda}\in H$ that minimizes the functional
\begin{equation}\label{E:Elamdf}
E_{\lambda}(f)=\sum_{j=1}^n\left[y_j-f(p_j)\right]^2+{\lambda}\|f\|^2,
\end{equation}
where $p_1,\ldots, p_n$ are given points in ${\mathcal X}$, $y_1,\ldots, y_n$ are given values in $\mbr$, and $\lambda>0$ is a parameter.  

Our first goal in this section is to present a  geometric   proof  of the following   representer theorem widely used in support vector machine theory. The result  has its  roots in the work of Kimeldorf and Wahba \cites{Wahba1, Wahba2} (for example, \cite[Lemmas 2.1 and 2.2]{Wahba1}) in the context of splines; in this context it is also worth noting the work of de Boor and Lynch  \cite{deBLyn} where Hilbert space methods were used to study splines.

\begin{theorem}\label{T:svmminim} With notation and hypotheses as above, there is a unique $\hat{f}_{\lambda}\in H$ such that $E_{\lambda}(\hat{f}_{\lambda})$ is $\inf_{f\in H}E_{\lambda}(f)$. Moreover, $\hat{f}_{\lambda}$ is given explicitly by
\begin{equation}\label{E:f0wahba}
\hat{f}_{\lambda}=\sum_{i=1}^n c_i K_{p_i}
\end{equation}
where the vector $c\in\mbr^n$ is $(K_D+\lambda I_n)^{-1}y$, with $K_D$ being the $n\times n$ matrix with entries $[K_D]_{ij}=K(p_i,p_j)$ and $y=(y_1,\ldots, y_n)\in\mbr^n$.
\end{theorem}

\begin{proof}  It will be convenient to scale the inner-product $\la \cdot, \cdot \ra$ of $H$ by $\lambda$. Consequently, we denote by $H_{\lambda}$ the space $H$ with inner-product:
\begin{equation}
\la f,g \ra_{H_{\lambda}}=\lambda\la f,g \ra \text{, for all } f,g\in H.
\end{equation}
We shall use the linear mapping
\begin{equation}\label{E:defT}
T:\mbr^n\to H_{\lambda}
\end{equation}
that maps $e_i$ to $\lambda^{-1} K_{p_i}$ for $i\in\{1,\ldots, n\}$, where $\{e_1,\ldots, e_n\}$ is the standard basis of $\mbr^n$:
$$T(e_i)=\lambda^{-1}K_{p_i}\qquad\hbox{for $i\in\{1,\ldots, n\}$.}
$$
We observe then that for any $f\in H_{\lambda}$
\begin{equation}
\la T^*f, e_i\rangle_{\mbr^n}=\la f, Te_i\ra_{H_{\lambda}}=\lambda\la f, \lambda^{-1}K_{p_i}\ra=f(p_i)
\end{equation}
for each $i$, and so
\begin{equation}\label{E:Tstarf}
T^*f=\sum_{j=1}^nf(p_j)e_j.
\end{equation}
Consequently, we can rewrite $E_{\lambda}(f)$ as
\begin{equation}\label{E:Ef}
E_{\lambda}(f)=\|y-T^*f\|_{\mbr^n}^2+ \|f\|_{H_{\lambda}}^2,
\end{equation}
and from this we see that $E_{\lambda}(f)$ has a geometric meaning as the distance from the point $(f, T^*f)\in H_{\lambda}\oplus \mbr^n $ to the point $(0,y)$ in $H_{\lambda}\oplus \mbr^n$:
\begin{equation}\label{E:Ef2}
E_{\lambda}(f)=\|(0,y)-(f,T^*f)\|_{H_{\lambda}\oplus \mbr^n}^2.
\end{equation}
Thus the minimization problem for $E_{\lambda}(\cdot)$ is equivalent to finding the point on the subspace
\begin{equation}
M=\{(f,T^*f)\,:\,f\in H_{\lambda}\}\subset H_{\lambda}\oplus\mbr^n
\end{equation}
closest to $(0,y)$. Now the subspace $M$ is just the graph  ${\rm Gr}(T^*)$ and it   is a closed subspace of $H_{\lambda}\oplus\mbr^n$ because it is the orthogonal complement of a subspace (as we see below in \eqref{E:GrTperp}). Hence by standard Hilbert space theory there is indeed a unique point on $M$ that is closest to $(0,y)$, and this point is in fact of the form
\begin{equation}\label{E:abf0}
(\hat{f}_{\lambda},T^*\hat{f}_{\lambda})=(0,y)+(a,b),
\end{equation}
where the vector $(a,b)\in H_{\lambda}\oplus \mbr^n$ is orthogonal to $M$. Now the condition for orthogonality to $M$ means that
$$\la (a,b), (f,T^*f)\ra_{H_{\lambda}\oplus \mbr^n}=0\qquad\hbox{for all $f\in H$,}$$
and this is equivalent to
$$0= \la a, f\ra_{H_{\lambda}}+\la b, T^*f\ra_{\mbr^n}=\la a+Tb,f\ra_{H_{\lambda}} = \lambda \la a+Tb, f\ra$$
for all $f\in H$. Therefore
\begin{equation}
a+Tb=0.
\end{equation}
Thus,
\begin{equation}\label{E:GrTstarperp}
\left[{\rm Gr}(T^*)\right]^\perp=\{(-Tc,c)\,:\,c\in\mbr^n\}.
\end{equation}
Conversely, we can check directly that
\begin{equation}\label{E:GrTperp}
 {\rm Gr}(T^*) =\{(-Tc,c)\,:\,c\in\mbr^n\}^\perp.
\end{equation}

\begin{figure}\label{Fig}
\centering
\begin{minipage}{.5\textwidth}
  \centering
  \includegraphics[width=\linewidth]{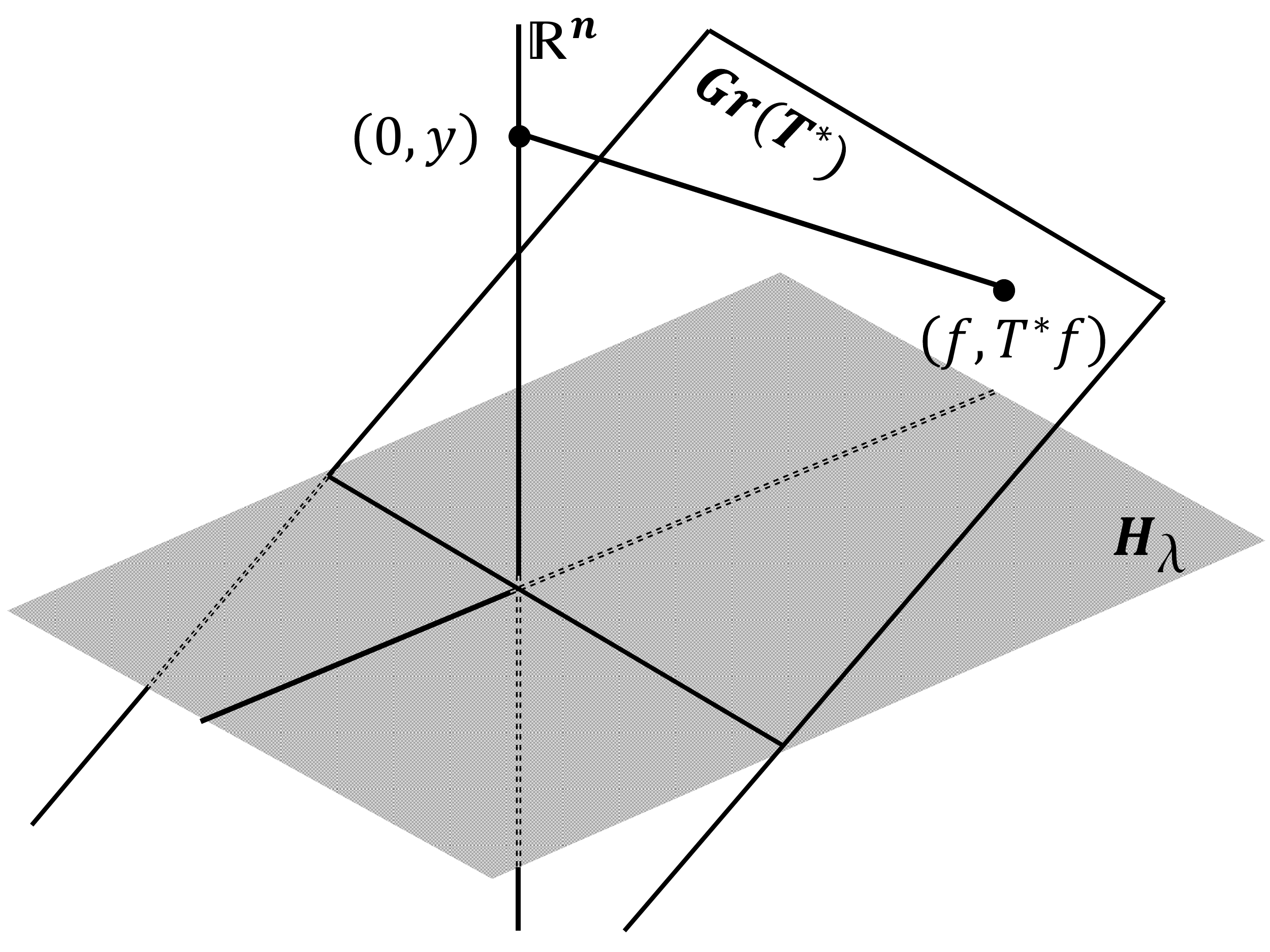}
\end{minipage}%
\begin{minipage}{.5\textwidth}
  \centering
  \includegraphics[width=\linewidth]{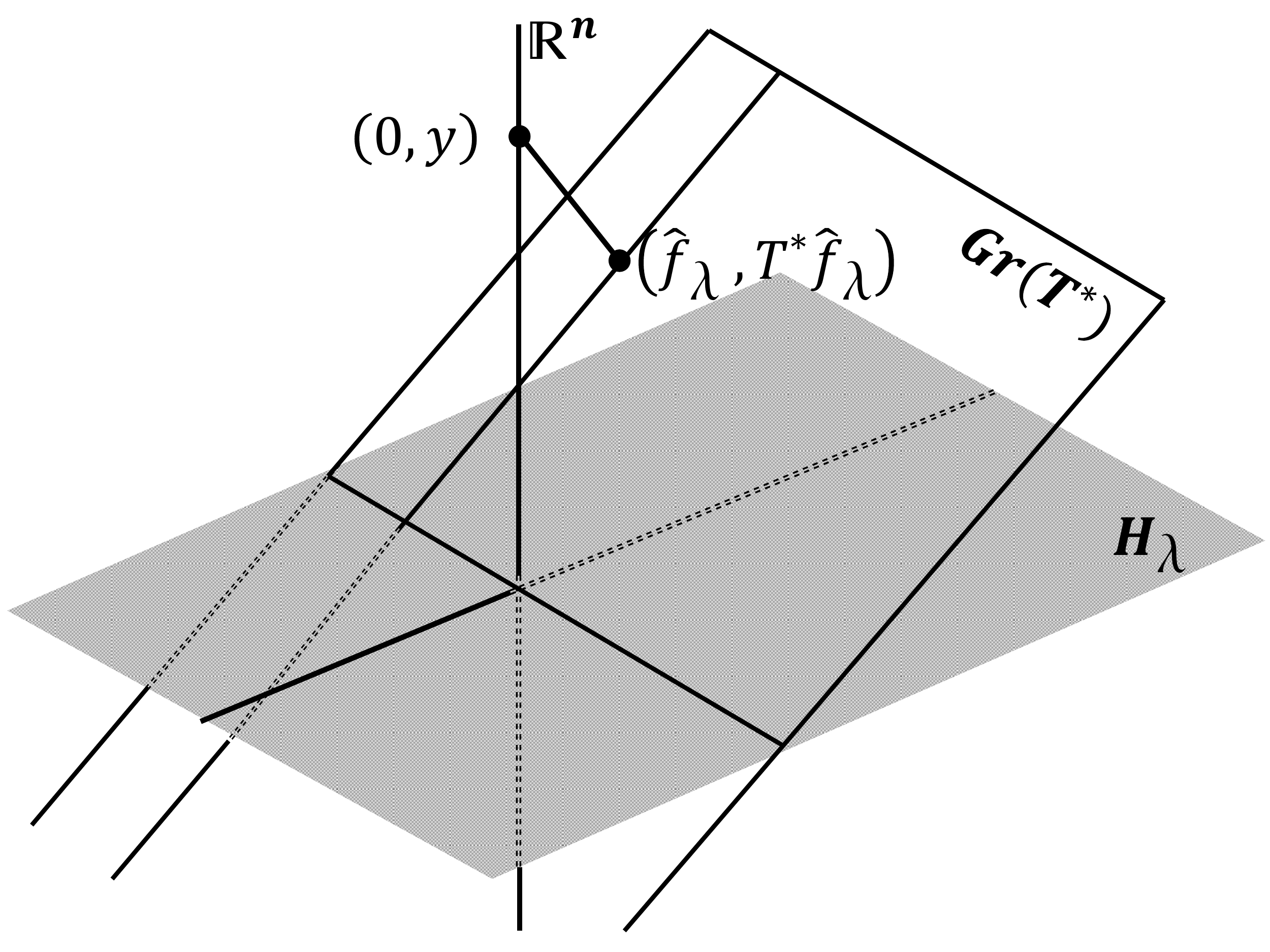}
\end{minipage}
\caption{\small{A geometric interpretation of Theorem \ref{T:svmminim}}}
\end{figure}

Returning to (\ref{E:abf0}) we see that the point on $M$ closest to $(0,y)$ is of the form
\begin{equation}\label{E:f0abTab}
(\hat{f}_{\lambda},T^*\hat{f}_{\lambda})=(a,y+b)= (-Tb,y+b)
\end{equation}
for some $b\in\mbr^n$. Since the second component is $T^*$ applied to the first, we have
$$y+b=-T^*Tb,$$
and solving for $b$ we obtain
\begin{equation}\label{E:bTstarTy}
b=-(T^*T+I_n)^{-1}y.
\end{equation}
Note that the operator $T^*T+I_n$ on $\mbr^n$  is invertible, since  $\la(T^*T+I_n)w,w\ra_{\mbr^n}\geq \|w\|_{\mbr^n}^2$, so that if $(T^*T+I_n)w=0$ then $w=0$.
Then from \eqref{E:f0abTab} we have $\hat{f}_{\lambda}=a=-Tb$ given by
\begin{equation}\label{E:f0Tb}
\hat{f}_{\lambda}=T\left[(T^*T+I_n)^{-1}y \right].
\end{equation}
Now we just need to write this in coordinates. The matrix for $T^*T$ has entries
\begin{equation}
\begin{split}
\la (T^*T)e_i,e_j\ra_{\mbr^n} &=\la Te_i,Te_j\ra_{\mbr^n}\\
&= \la \lambda^{-1}K_{p_i}, \lambda^{-1}K_{p_j}\ra_{H_{\lambda}}\\
&= \lambda\la \lambda^{-1}K_{p_i}, \lambda^{-1}K_{p_j}\ra \\
&=\lambda^{-1}[K_D]_{ij},
\end{split}
\end{equation}
and so
$$\hat{f}_{\lambda}=T\left[(T^*T+I_n)^{-1}y \right]=T\left[\sum_{i,j=1}^n(I_n+\lambda^{-1}K_D)^{-1}_{ij}y_je_i\right].$$
Since $Te_i=\lambda^{-1}K_{p_i}$, we can write this as
$$\hat{f}_{\lambda}=\sum_{i=1}^n\left[\sum_{j=1}^n(I_n+\lambda^{-1}K_D)^{-1}_{ij}y_j\right]\lambda^{-1}K_{p_i},$$
which simplifies readily to (\ref{E:f0wahba}). \end{proof}

The  observations  about the graph ${\rm Gr}(T^*)$ used in the preceding proof are in the spirit of  the analysis of adjoints of operators  carried out by von Neumann \cite{JvN1932}. 

With $\hat{f}_{\lambda}$ being the minimizer as above, we can calculate the minimum value of $E(\cdot)$: 
\begin{equation}\label{E:Ef0min}
\begin{split}
E(\hat{f}_{\lambda}) &= \|(0,y) -(\hat{f}_{\lambda},T^*\hat{f}_{\lambda})\|_{H_{\lambda}\oplus \mbr^n}^2\\
&=\|(a,b)\|_{H_{\lambda}\oplus \mbr^n}^2\qquad\hbox{(using \eqref{E:abf0})}\\
&=\la Tb,Tb\ra_{H_{\lambda}}+\la b,b\ra_{\mbr^n}\\
&=\la T^*Tb,b\ra_{\mbr^n}+\la b,b\ra_{\mbr^n}\\
&=\la (T^*T +I_n)b,b\ra_{\mbr^n}\\
&=\la -y, b \ra_{\mbr^n}\qquad\hbox{(using \eqref{E:bTstarTy})}\\
&= \| (T^*T+I)^{-1/2}y\|_{\mbr^n}^2\qquad\hbox{(again using (\ref{E:bTstarTy}).)}
\end{split}
\end{equation}
It is useful to keep in mind that our definition of $T$ in (\ref{E:defT}), and hence of $T^*$, depends on $\lambda$.   
We note that the norm squared of ${\hat f}_{\lambda}$ itself is
\begin{equation}\label{E:fhatlamdanorm}
\begin{split}
\|{\hat f}_{\lambda}\|^2&=\lambda^{-1}\la T(T^*T+I_n)^{-1}y, T(T^*T+I_n)^{-1}y\ra_{H_\lambda}  \\
&= \lambda^{-1}\| (T^*T)^{1/2}(T^*T+I)^{-1}y\|_{\mbr^n}^2\\
\end{split}
\end{equation}

Let us now turn to the traditional spline setting. A function $f\in H$, whose graph passes through the training points $(p_i,y_i)$, for $i\in\{1,\ldots, n\}$, of minimum norm has to be found.  We present here a geometrical description in the spirit of Theorem \ref{T:svmminim}. This is in fact the result one would obtain by formally taking   $\lambda=0$ in  Theorem \ref{T:svmminim}. 

\begin{theorem}\label{T:splinemin} Let $H$ be a reproducing kernel Hilbert space of functions on a set $\mathcal X$, and let $(p_1,y_1),\ldots, (p_n,y_n)$ be points in ${\mathcal X}\times{\mbr}$. Let $K:{\mathcal X}\times {\mathcal X}\to\mbr$ be the reproducing kernel for $H$, and let $K_q:{\mathcal X}\to\mbr:q'\mapsto K(q,q')$, for every $q\in {\mathcal X}$.  Assume that the functions $K_{p_1},\ldots, K_{p_n}$ are linearly independent. Then, for any $y=(y_1,\ldots, y_n)\in\mbr^n$,  the element in 
$$\{f\in H\,:\,f(p_1)=y_1,\ldots, f(p_n)=y_n\}$$
 of minimum norm is given by
\begin{equation}\label{E:hatf0}
{\hat f}_0=\sum_{i=1}^nc_{0i}K_{p_i},
\end{equation}
where $c_0=(c_1,\ldots, c_n)=K_D^{-1}y$, with $K_D$ being the $n\times n$ matrix whose $(i,j)$-th entry  is $K(p_i,p_j)$.
\end{theorem}

The assumption of linear independence of the $K_{p_i}$ is simply to ensure that there does exist a function $f\in H$ with values $y_i$ at the points $p_i$.
\begin{proof} 
Let $T_0:\mbr^n\to H$ be the linear mapping specified by $T_0e_i=K_{p_i}$, for $i\in\{1,\ldots, n\}$.  Then  the adjoint $T^*:H\to\mbr^n$ is given explicitly by
$$T_0^*f=\sum_{i=1}^n\la f,K_{p_i}\ra e_i=\sum_{i=1}^nf(p_i)e_i$$
and so 
\begin{equation} \label{E:clsubsp}
 \{f\in H\,:\,f(p_1)=y_1,\ldots, f(p_n)=y_n\}=\{f\in H\,:\,T_0^*f=y\}.
 \end{equation}
 Since the linear functionals $K_{p_i}$ are linearly independent, no nontrivial linear combination of them is $0$ and so the only vector in $\mbr^n$ orthogonal to the range of $T_0^*$ is $0$; thus
 \begin{equation}
 {\rm Im}(T_0^*)=\mbr^n.
 \end{equation}
Let  ${\hat f}_0$ be the point on the closed affine subspace in \eqref{E:clsubsp} that is nearest the origin in $H$. Then ${\hat f}_0$ is the point on $\{f\in H\,:\,T_0^*f=y\}$ orthogonal to $\ker T_0^*$. 

\begin{figure}[h]
\centering
\includegraphics[scale=0.3]{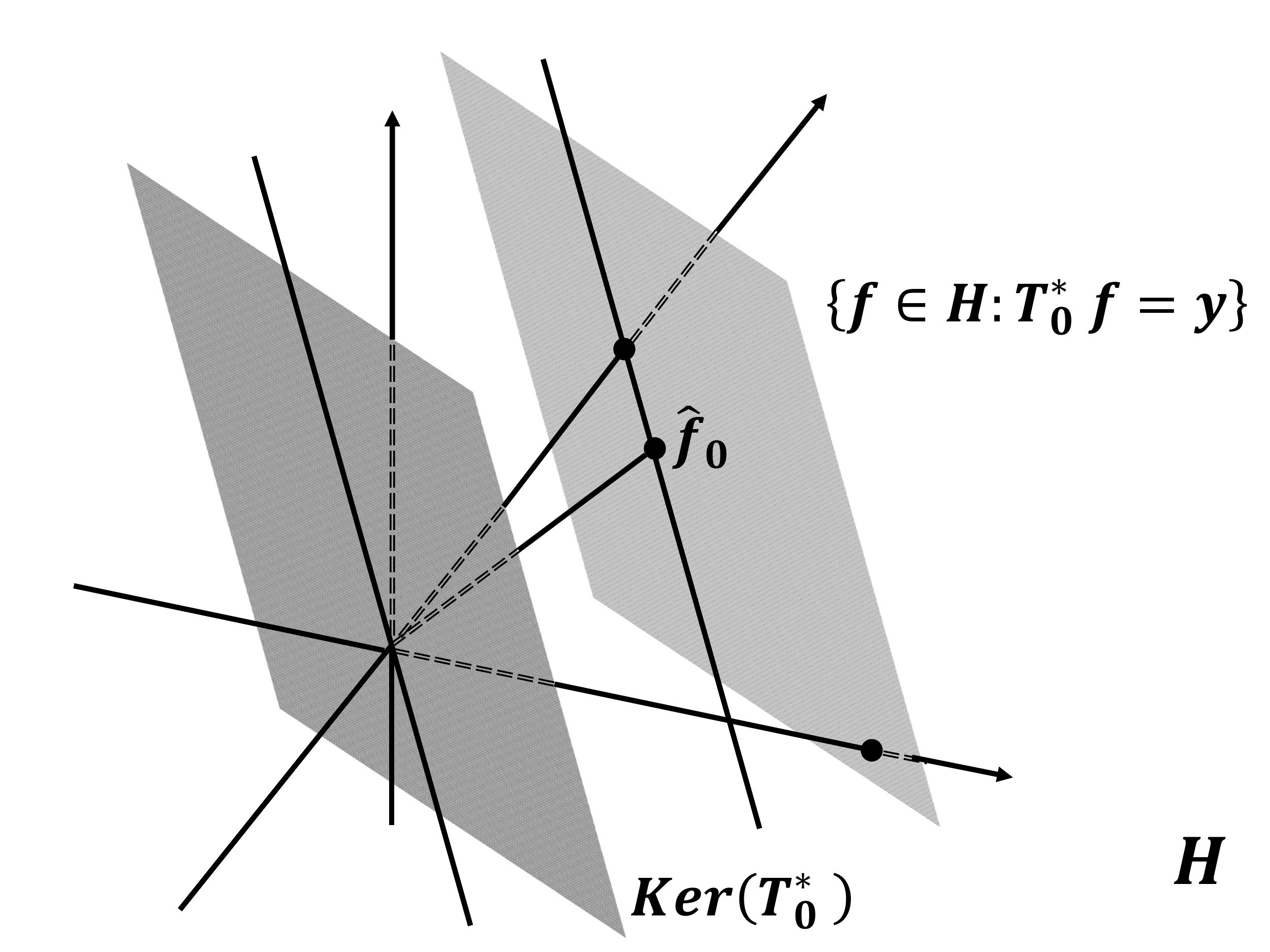}
\caption{\small{The point on $\{f\in H\,:\,T_0^*f=y\}$ closest to the origin.}}
\end{figure}

Now it is a standard observation that
  $$\ker T_0^*=\left[{\rm Im}(T_0)\right]^\perp.$$
 (If $v\in\ker(T_0^*)$ then $\la v,T_0w\ra=\la T_0^*v,w\ra=0$, so that $\ker T_0^*\subset \left[{\rm Im}(T_0)\right]^\perp$; conversely, if $v\in \left[{\rm Im}(T_0)\right]^\perp$ then $\la T_0^*v,w\ra=\la v, T_0w\ra=0$ for all $w\in H$ and so $T_0^*v=0$.)  Hence $\left(\ker T_0^*\right)^\perp$ is the closure of ${\rm Im}(T_0)$. Now ${\rm Im}(T_0)$ is a finite-dimensional subspace of $H$ and hence is closed; therefore
  \begin{equation}
\left( \ker T_0^*\right)^\perp= {\rm Im}(T_0).
 \end{equation}
 Returning to our point ${\hat f}_0$ we conclude that ${\hat f}_0\in {\rm Im}(T_0)$. Thus, ${\hat f}_0=T_0g$ for some $g\in H$.  The requirement that $T_0^*{\hat f}_0$ be equal to $y$ means that $(T_0^*T_0)g=y$, and so
 \begin{equation}\label{E:hatf0T0T}
 {\hat f}_0=T_0(T_0^*T_0)^{-1}y.
 \end{equation}
We observe here that $T_0^*T_0:\mbr^n\to\mbr^n$ is invertible because any $v\in \ker(T_0^*T_0)$ satisfies $\|T_0v\|^2=\la T_0^*T_0v,v\ra=0$, so that $\ker(T_0^*T_0)=\ker(T_0)$, and this is $\{0\}$, again by the linear independence of the functions $K_{p_i}$.
 The matrix for $T_0^*T_0$ is just the matrix $K_D$ because its   $(i,j)$-th entry is
 $$(T_0^*T_0)_{ij}=\la T_0 e_i,T_0 e_j\ra=\la K_{p_i}, K_{p_j}\ra=K(p_i,p_j)=(K_D)_{ij}.$$
 Thus, using (\ref{E:hatf0T0T}), ${\hat f}_0$ works out to $\sum_{i,j=1}^n(K_D^{-1})_{ij}y_jK_{p_i}$. \end{proof}

Working in the setting of Theorem \ref{T:splinemin},  and assuming that $H$ is separable, let $B$ be the Banach space obtained as completion of $H$ with respect to a measurable norm.  Recall from (\ref{E:GRcharIh}) that the  Gaussian Radon transform of a function $F$ on $H$ is the function on the set of closed affine subspaces of $H$ given by 
\begin{equation}\label{E:defGFLmuL}
GF(L)=\int F\,d\mu_{L},
\end{equation}
where   $L$ is any closed affine subspace of $H$, and $\mu_L$ is the Borel measure on $B$ specified by the Fourier transform
\begin{equation}\label{E:ftmuL}
\int_B e^{iI(h)}\,d\mu_L=e^{i\la h, {\hat f}_0\ra-\frac{1}{2}\|Ph\|^2},\quad\hbox{for all $h\in H$,}
\end{equation}
wherein ${\hat f}_0$ is the point on $L$ closest to the origin in $H$ and $P$ is the orthogonal projection onto the closed subspace $L-{\hat f}_0$. Let us apply this to the closed affine subspace
$$L=\{f\in H\,:\,f(p_1)=y_1,\ldots, f(p_n)=y_n\}.$$
From equation (\ref{E:ftmuL}) we see that  $I(h)$ is a Gaussian variable with mean $\la h, {\hat f}_0\ra$ and variance $\|Ph\|^2$. Now let us take for $h$ the function $K_p\in H$, for any point $p\in {\mathcal X}$; then  $I(K_p)$ is Gaussian, with respect to the measure $\mu_L$,  with mean 
\begin{equation}
\mbe_{\mu_L}({\tilde K}_p)= \la K_p, {\hat f}_0\ra ={\hat f}_0(p),
\end{equation}
where ${\tilde K}_p$ is the random variable $I(K_p)$ defined on $(B,\mu_L)$. The function ${\hat f}_0$ is as given in (\ref{E:hatf0}).  Now consider the special case where $p=p_i$ from the training data. Then 
$$\mbe_{\mu_L}({\tilde K}_{p_i})={\hat f}_0(p_i)=y_i,$$
because ${\hat f}_0$ belongs to $L$, by definition. Moreover, the variance of ${\tilde K}_{p_i}$ is the norm-squared of the orthogonal projection of ${\tilde K}_{p_i}$ onto the closed subspace
$$L_0=L-{\hat f}_0=\{f\in H\,:f(p_1)=0,\ldots, f(p_n)=0\}.$$
However, for any $f\in L$ we have 
$$\la f-{\hat f}_0, {\tilde K}_{p_i}\ra = f(p_i)-{\hat f}_0(p_i)=0,$$
and so the variance of ${\tilde K}_{p_i}$ is $0$. Thus, with respect to the measure $\mu_L$, the functions ${\tilde K}_{p_i}$ take the constant values $y_i$ almost everywhere. This is  analogous to our result Theorem \ref{T:fhatascp}; in the present context the conclusion is
\begin{equation}\label{E:fhatcp0}
{\hat f}_0(p)=\mbe\left[{\tilde K}_p\,|\, {\tilde K}_{p_1}=y_1,\ldots, {\tilde K}_{p_1}=y_1\right].
\end{equation}

\section{Proof of Lemma \ref{L:condExpGauss}}

The following result is a standard one, but we include a proof here for completeness.

\begin{lemma}\label{L:condExpGauss}
Suppose that $(Z, Z_1, Z_2, \ldots, Z_n)$ is a centered $\mbr^{n+1}$-valued Gaussian random variable  on a probability space $(\Omega, \mathcal{F}, \mathbb{P})$ and let $A \in \mathbb{R}^{n\times n}$ be the covariance matrix:
	\begin{equation*}
	A_{jk} = {\rm Cov}(Z_j, Z_k) \text{, for all } 1 \leq j,k \leq n,
	\end{equation*}
	and suppose that $A$ is invertible.
Then:
	\begin{equation}\label{E:condExpGauss}
	\mathbb{E}\left[ Z | Z_1, Z_2, \ldots, Z_n \right] = a_1Z_1 + a_2Z_2 + \ldots + a_nZ_n
	\end{equation}
where $a = (a_1, \ldots, a_n)\in \mathbb{R}^n$ is given by 
\begin{equation}\label{E:aAz}
a = A^{-1}z,
\end{equation}
where $z\in \mathbb{R}^n$ is given by $z_k =  {\rm Cov}(Z, Z_k)$ for all $1\leq k \leq n$.
\end{lemma}

\begin{proof}  Let $Z_0$ be the orthogonal projection of  $Z$ on the linear span of $Z_1,\ldots, Z_n$; thus  $Y=Z-Z_0$ is orthogonal to $Z_1,\ldots, Z_n$ and, of course, $(Y, Z_1,\ldots, Z_n)$ is Gaussian.  Hence, being all jointly Gaussian, the random variable $Y$ is independent of $(Z_1,\ldots, Z_n)$.  Then for any $S\in\sigma(Z_1,\ldots, Z_n)$ we have
\begin{equation}\label{E:ZZ1n}
\begin{split}
\mbe[Z1_S] &=\mbe[Z_01_S]+\mbe[Y1_S]\\
&=\mbe[Z_01_S]+\mbe[Y]\mbe[1_S]\\
&=\mbe[Z_01_S].
\end{split}
\end{equation}
Since this holds for all $S\in \sigma(Z_1,\ldots, Z_n)$, and since the random variable $Z_0$, being a linear combination of $Z_1,\ldots, Z_n$, is $\sigma(Z_1,\ldots, Z_n)$-measurable, we conclude that
\begin{equation}
Z_0=\mbe\left[Z\,|\,\sigma(Z_1,\ldots, Z_n)\right].
\end{equation}
Thus the conditional expectation of $Z$ is the orthogonal projection $Z_0$ onto the {\em linear span} of the variables $Z_1,\ldots, Z_n$.

Writing
	\begin{equation*}
	Z_0 = a_1Z_1 + a_2Z_2 + \ldots + a_nZ_n,
	\end{equation*}
	we have
	$$\mbe[ZZ_j]=\mbe[Z_0Z_j]=\sum_{k=1}^n\mbe[Z_jZ_k]a_k=(Aa)_j,$$
	noting that $A_{jk}={\rm Cov}(Z_j,Z_k)=\mbe[Z_jZ_k]$ since all these variables have mean $0$ by hypothesis.
Hence we have $a=A^{-1}z$.
\end{proof}

{\bf Acknowledgments}. This work is part of a research project covered by  NSA grant
  H98230-13-1-0210. I. Holmes would like to express her gratitude to the Louisiana State University Graduate School for awarding her the LSU Graduate School Dissertation Year Fellowship, which made most of her contribution to this work possible. We thank Kalyan B. Sinha for useful discussions.

\end{document}